%% file: main.tex
\documentclass[]{custom}
\usepackage{float}
\usepackage{graphicx}
\usepackage{titletoc}

\usepackage[english]{babel}
\usepackage{xcolor}
\usepackage{color}
\usepackage{wrapfig}
\usepackage[table]{xcolor}

\usepackage{amsmath}
\usepackage{amssymb}
\usepackage{amsthm}
\usepackage{graphicx}
\usepackage{booktabs}      
\usepackage{array}
\usepackage{multirow}      
\usepackage{caption}  
\usepackage{amsmath}
\usepackage[table]{xcolor}  
\usepackage{xcolor}
\usepackage{microtype}
\usepackage{parskip}

\definecolor{bonebg}{HTML}{F7F3EC}
\definecolor{accentgreen}{HTML}{6E8F80}
\definecolor{juniper}{HTML}{4D6A5D}
\definecolor{mist}{HTML}{EEF1EF}
\definecolor{panelborder}{HTML}{A9B3AE}
\definecolor{muted}{HTML}{666666}
\definecolor{badred}{HTML}{C83A2A}

\newtcolorbox{figurepanel}{
    enhanced,
    colback=mist,
    colframe=panelborder,
    boxrule=0.8pt,
    arc=6pt,
    left=10pt,right=10pt,top=10pt,bottom=10pt
}

\newtcolorbox{samplebox}{
    enhanced,
    colback=accentgreen!4!bonebg!60!white,
    colframe=accentgreen!55!black,
    boxrule=0.6pt,
    arc=4pt,
    left=8pt,right=8pt,top=7pt,bottom=7pt
}

\newcommand{\nfeheader}[1]{%
    {\sffamily\bfseries\large\color{juniper} #1}
}
\newcommand{\methodtitle}[1]{%
    {\bfseries #1}
}
\newcommand{\metrics}[2]{%
    {\color{muted} Gen.\ PPL: \textbf{#1} \\
    \hfill Entropy: \textbf{#2}}
}

\usepackage{algorithm}
\usepackage{algpseudocode}

\include{math_commands}

\usepackage{thmtools}
\usepackage{thm-restate}

\newtcolorbox{propbox}{
    enhanced,
    colback=accentgreen!4!bonebg!60!white,      
    colframe=accentgreen, 
    boxrule=0pt,
    leftrule=5pt,      
    arc=0pt,            
    left=8pt, right=8pt, top=8pt, bottom=8pt,
    fonttitle=\bfseries\sffamily,
    coltitle=juniper
}

\title{\fontsize{0.69cm}{5.5cm}\selectfont Discrete Flow Maps}

\author[1,2]{Peter Potaptchik}
\author[3]{\, Jason Yim}
\author[2]{\, Adhi Saravanan}
\author[4]{\, Peter Holderrieth}
\author[5]{\\Eric Vanden-Eijnden}
\author[1,6]{\, Michael S. Albergo}

\affiliation[1]{Harvard University}
\affiliation[2]{University of Oxford}
\affiliation[3]{Independent}
\affiliation[4]{MIT}
\affiliation[5]{NYU}
\affiliation[6]{Kempner Institute}

\abstract{\input{sections/abstract}}
\begin{document}
\maketitle

\renewcommand{\footnoterule}{%
  \kern -3pt
  \hrule width \linewidth
  \kern 2.6pt
}
\vspace{-0.2cm}

\input{sections/body}

\bibliographystyle{acl_natbib.bst}
\bibliography{bibliography.bib}

\clearpage
\beginappendix
\startcontents[app]
\printcontents[app]{l}{1}{\setcounter{tocdepth}{2}}

\input{sections/appendix}

\end{document}

%% file: math_commands.tex
\usepackage{amsmath,amsfonts,bm}

\def\1{\bm{1}}

\DeclareMathAlphabet{\mathsfit}{\encodingdefault}{\sfdefault}{m}{sl}
\SetMathAlphabet{\mathsfit}{bold}{\encodingdefault}{\sfdefault}{bx}{n}

\newcommand{\KL}{D_{\mathrm{KL}}}



%% file: sections/abstract.tex
The sequential nature of autoregressive next-token prediction imposes a fundamental speed limit on large language models. While continuous flow models offer a path to parallel generation, they traditionally demand expensive iterative integration. Flow Maps bypass this bottleneck by compressing generative trajectories into single-step mappings, theoretically enabling the generation of full text sequences from noise in a single forward pass. However, standard formulations rely on Euclidean regression losses that are geometrically ill-suited for discrete data. In this work, we resolve this conflict with Discrete Flow Maps, a framework that reconciles trajectory compression with the geometry of the probability simplex. We recast standard flow map training for the discrete domain, aligning the training dynamics with the discrete nature of language. Empirically, this strict geometric alignment allows our method to surpass previous state-of-the-art results in discrete flow modeling.

%% file: sections/body.tex
\section{Introduction}

\begin{wrapfigure}[22]{r}{0.46\textwidth}
	\vspace{-1.1em}
	\centering
	\includegraphics[width=\linewidth]{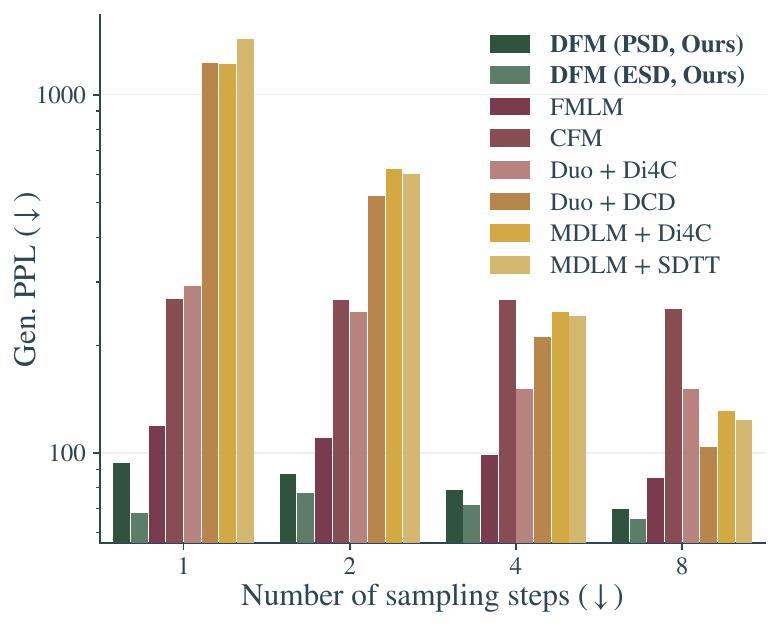}
	\caption{Generative perplexity as a function of the number of sampling steps on the LM1B dataset, comparing Discrete Flow Maps (DFM) with other accelerated methods. This highlights the ability of DFMs to generate higher-quality text with fewer sampling steps.}
	\label{fig:pareto}
\end{wrapfigure}

In just a few years, large language models (LLMs) \citep{vaswani2017attention, brown2020language,chowdhery2022palm,touvron2023llama} have become general-purpose engines for text, code, and multimodal reasoning —powering a variety of applications ranging from chat assistants and search to programming tools and scientific discovery. Yet, despite their remarkable practical impact, the field remains constrained by a structural bottleneck: the inherently sequential nature of next-token prediction. The dominant architecture, autoregressive (AR) modeling, generates text one step at a time. While this approach has scaled remarkably, it imposes a linear computational cost on generation, rendering long-form reasoning and real-time synthesis  expensive. Various powerful optimization techniques have been proposed---such as speculative decoding \citep{leviathan2023fast} and multi-token prediction \citep{gloeckle2024better}---aimed at extracting efficiency from the AR backbone. However, these methods remain strictly bound by the underlying serial nature of AR models. To fundamentally overcome this limitation, we look beyond the next-token prediction paradigm entirely.

Separately, diffusion models \citep{song2020score, ho2020denoising, sohldickstein2015deep} and flow matching \citep{albergo2022building, lipman2022flow,liu2022flow} have emerged as the leading approaches for generative synthesis in continuous domains. By modeling generation as the transformation of noise into data via differential equations, these frameworks offer a rigorous path toward non-autoregressive, parallel generation. Crucially, they unlock capabilities that AR models lack, such as precise test-time steering \citep{singhal2025general,uehara2025inference} and flexible guidance mechanisms \citep{chung2022diffusion}. To accelerate these models, recent techniques such as consistency models and flow maps \citep{boffi_flow_2024, boffi2025consistency} have emerged, learning to map any point on a generative trajectory directly to its endpoint, thereby compressing the iterative integration process into a single forward pass. The recent surge in such distillation methods suggests a tantalizing possibility: training flow maps on text to achieve massive speedups while retaining powerful control mechanisms.

However, a fundamental geometric mismatch stands in the way. Standard flow map objectives are designed for Euclidean space $\mathbb{R}^K$, relying on $L^2$ regression losses. Text, in contrast, is discrete: the natural object to predict is a probability distribution over a vocabulary, which lives on the probability simplex---not in Euclidean space. For such discrete data,  $L^2$ regression losses yield suboptimal performance compared to likelihood-based alternatives such as cross-entropy. Simply put, treating a probability distribution like a coordinate in Euclidean space is a fundamental misalignment; the geometry of the loss does not match the geometry of the data.

In this work, we resolve this conflict by systematically recasting continuous flow maps for discrete data. Rather than parameterizing flow maps in terms of Euclidean quantities, we reparameterize them in terms of the mean denoiser, an object that natively lives on the probability simplex. This lets us replace the Euclidean objectives used in standard flow map training with exact cross-entropy and KL divergence losses, which are natural for discrete data. This yields a geometrically consistent framework that also performs strongly in practice: it yields a cleaner formulation of discrete flow maps while preserving the ability to perform one and few-step generation. This geometric consistency is not merely aesthetic, but translates into stronger results, allowing us to surpass previous state-of-the-art performance in non-autoregressive language generation.

To summarize, we make the following contributions:
\begin{enumerate}
\item \textbf{Discrete Flow Maps:} We provide a paradigm for one or few-step non-autoregressive text generation by generalizing flow map models to discrete data. Remarkably, this reparameterization is fully defined by a \textit{mean denoiser} that natively lives on the probability simplex.
\item \textbf{Training objectives:} We exploit the above relation to derive cross-entropy and Kullback-Leibler (KL) divergence losses for any-step flow maps in terms of the mean denoiser. These relations open a wide design space of objective functions that we explore to best align with the geometry of the data.
\item \textbf{Experiments:} We show that Discrete Flow Maps enable substantial speedups for language modeling, supporting one and few-step generation with only minor performance degradation, while also allowing for test-time steering and guidance.
\end{enumerate}

\begin{figure}[t]
    \centering
    \includegraphics[width=1.0\linewidth]{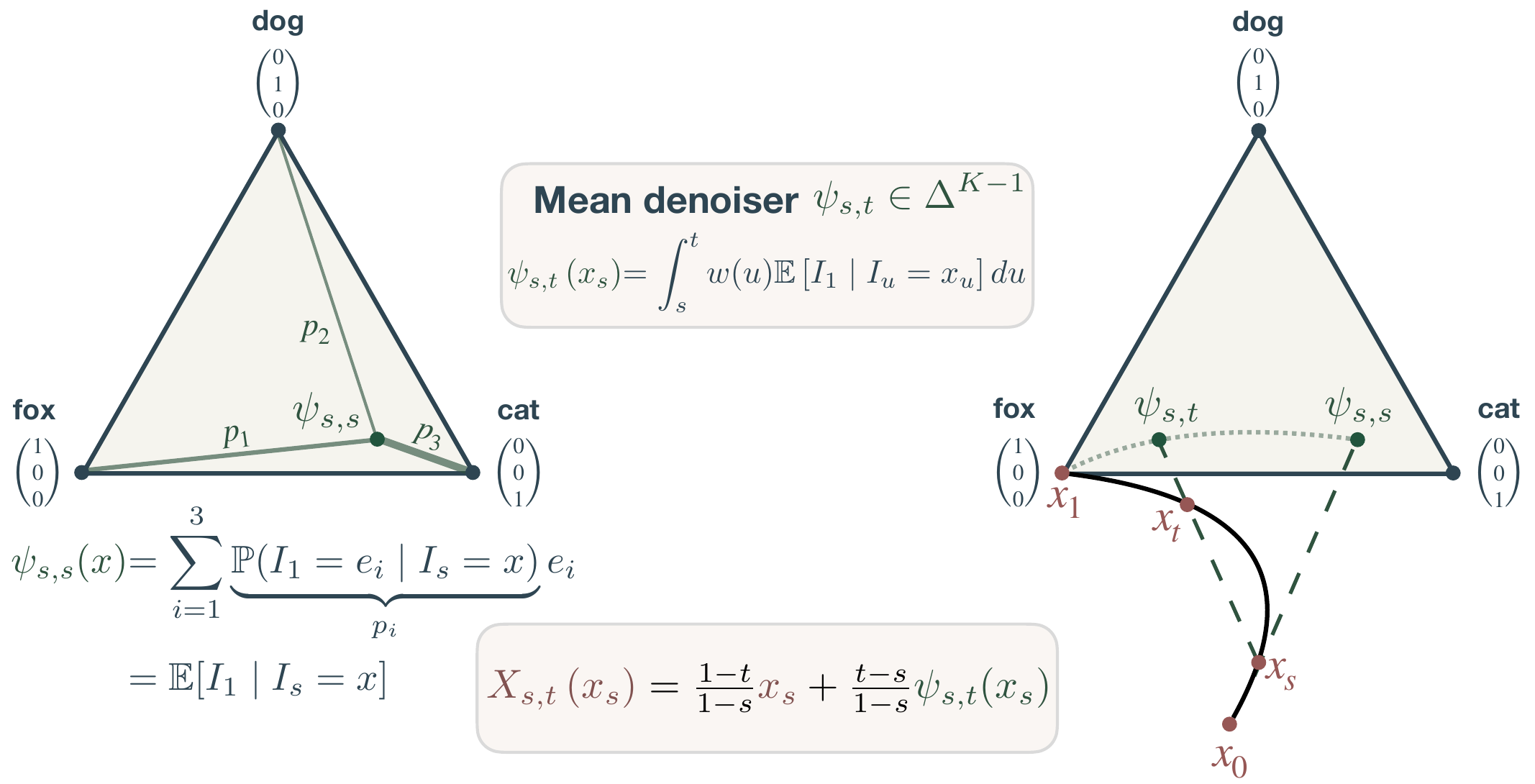}
    \caption{\textbf{Overview of the geometry of the discrete flow map.} \textbf{Left:} The instantaneous denoiser $\psi_{s,s}(x_s)$ at any time $s$ is always on the simplex by weighted sum. \textbf{Right:} By derivation of \eqref{eq:mean-denoiser}, the \textbf{mean denoiser} $\psi_{s,t}(x_s)$ is also always a projection onto the simplex, making possible a direct parameterization of the flow map $X_{s,t}$ that enables our cross-entropy loss functions.}
    \label{fig:placeholder}
\end{figure}

\section{Preliminaries}
\label{sec:prelim}

We begin by formalizing language modeling within the geometry of the probability simplex, alongside the standard autoregressive formulation. We then review continuous generative flows that transport probability mass from noise to data and admit efficient acceleration via flow maps that compress trajectories into single-step operators.

\subsection{Language Modeling on the Simplex}
\label{sec:lang_model}
Let $\mathcal{V} = \{e_1, \dots, e_K\} \subset \mathbb{R}^K$ be the set of standard basis vectors representing a finite vocabulary of size $K$. These vectors constitute the vertices of the $(K-1)$-dimensional probability simplex $\Delta^{K-1} := \{x \in \mathbb{R}^K : x \geq 0, \langle \mathbf{1}, x \rangle = 1\}$. Consequently, a sequence of discrete tokens of length $L$ can be represented as a matrix $\mathbf{x} = (x^1, \dots, x^L) \in \mathcal{V}^L$. The objective of language modeling is to learn a probability distribution supported on these discrete sequences $\mathcal{V}^L$ that approximates the true data distribution, which we denote by $p_1(\mathbf{x})$.

\paragraph{Autoregressive Modeling.}
The standard Autoregressive (AR) approach factorizes the joint distribution $p_1(\mathbf{x})$ into a product of conditional probabilities via the chain rule:
\begin{align}
\label{eq:ar_factorization}
    p_1(\mathbf{x}) = \prod_{\ell=1}^L p(x^{\ell} \mid x^{<\ell}),
\end{align}
where $x^{<\ell}$ denotes the tokens preceding position $\ell$. AR models parametrize these conditionals using a neural network that outputs a categorical distribution over $\mathcal{V}$ at each step. To learn these conditional probabilities, training minimizes the negative log-likelihood, or equivalently, the cross-entropy loss.

\paragraph{Cross-Entropy.}
This objective is fundamental to discrete generative modeling as it allows for learning a probability mass function over a finite vocabulary. Let $Y$ be a discrete random variable taking values in $\mathcal{V}  = \{e_1, \dots, e_K\}$, and let $X$ be a conditioning random variable taking values in a generic measurable space $\mathcal{X}$. We optimize over the space of measurable functions $f: \mathcal{X} \rightarrow \Delta^{K-1}$ to find the minimizer of the expected cross-entropy loss:
\begin{equation}
    \mathcal{L}(f) = \mathbb{E}_{X,Y} \left[ - \sum_{k=1}^{K} Y^{(k)} \log f^{(k)}(X) \right].
\end{equation}
The global minimizer $f^*$ of this objective is the \emph{conditional expectation} of the target:
\begin{equation}
    f^*(x) = \mathbb{E}[Y \mid X=x]
\end{equation}
Crucially, since $Y$ is a one-hot vector, its expectation corresponds exactly to the vector of class probabilities: $\mathbb{E}[Y^{(k)} \mid X=x] = \mathbb P(Y=e_k \mid X=x)$. Thus, minimizing cross-entropy directly recovers the true conditional probability mass function. In the context of AR modeling, identifying $X$ with the history $x^{<\ell}$ and $Y$ with the next token $x^{\ell}$, this recovers $p(x^{\ell} \mid x^{<\ell})$, the true conditional distribution of the next token given the previous ones. While cross-entropy provides a rigorous foundation for training discrete generative models, the standard AR framework is hampered by the sequential dependency in~\eqref{eq:ar_factorization}, necessitating slow serial generation. To overcome this computational bottleneck, we instead turn to continuous generative flows, which we formalize next.

\subsection{Continuous Generative Flows}
Here, we model the data distribution via a neural transport map from a source distribution $p_0$ to the data distribution $p_1$, realized via an Ordinary Differential Equation (ODE) defined by a velocity field (drift) $b_t : \mathbb{R}^K \to \mathbb{R}^K$:
\begin{equation}
\label{eq:ode}
    \dot{x}_t = b_t(x_t), \qquad x_0 \sim p_0,
\end{equation}
constructed such that the trajectory endpoint $x_1$ is distributed as $p_1$. To learn this transport, we first specify the desired evolution of marginal densities using a \emph{stochastic interpolant}. We define a process $I_t$ that linearly interpolates between a noise sample $I_0 \sim p_0$ and a data sample $I_1 \sim p_1$:
\begin{align}
\label{eq:interpolant}
    I_t = (1-t)I_0 + t I_1.
\end{align}
Although we restrict ourselves to the linear interpolant in the main paper, our framework extends to a general class of interpolants (see Appendix~\ref{sec:proofs}). This process defines a time-dependent density $p_t := \text{Law}(I_t)$ connecting $p_0$ to $p_1$. We seek a vector field $b_t$ such that the marginal path of the ODE~\eqref{eq:ode} matches the interpolant (i.e., $x_t \sim p_t$). The optimal choice for $b_t$ is the conditional expectation of the interpolant's velocity:
\begin{equation}
    b_t(x) = \mathbb{E}[I_1 - I_0 \mid I_t = x].
\end{equation}
To learn this drift, we parameterize a neural network $\hat{b}_t: \mathbb{R}^K \to \mathbb{R}^K$ and minimize the flow matching objective \citep{albergo2022building,lipman2022flow, liu2022flow}:
\begin{equation}
    b = \operatorname*{argmin}_{\hat{b}} 
    \int_0^1 \mathbb{E} \left [ \left\| \hat{b}_t(I_t) - (I_1 - I_0)\right \|^2 \right]dt,
\end{equation}
where the expectation is taken over the interpolant process.

\subsection{Flow Maps and Trajectory Compression}
Solving~\eqref{eq:ode} during inference requires numerical integration, necessitating numerous evaluations of the neural drift $b_t$. To circumvent this bottleneck, methods such as Consistency Models and Flow Maps~\citep{song2023consistencymodels,boffi_flow_2024, sabour2025align} compress these continuous trajectories into single-step and few-step mappings.

By definition, the flow map $X_{s,t}: \mathbb{R}^K \to \mathbb{R}^K$ is the solution operator for the probability flow ODE~\eqref{eq:ode}. For any solution $(x_t)_{t \in [0,1]}$ of this ODE, the map satisfies:
\begin{align}
\label{eq:map}
    X_{s,t}(x_s) = x_t, \quad \forall s,t \in [0,1].
\end{align}
In essence, $X_{s,t}$ jumps directly between times $s$ and $t$ along the flow. For training, it is standard to parametrize the flow map in residual form via the average velocity $v_{s,t}(x)$:
\begin{align}
\label{eq:map:resid}
    X_{s,t}(x) = x + (t - s)\, v_{s,t}(x).
\end{align}
For the flow map to remain consistent with the underlying ODE dynamics, the average velocity must converge to the instantaneous drift as the time step vanishes. This is formalized as the \emph{tangent condition}~\citep{kim_consistency_2024}:
\begin{align}
\label{eq:tangent}
    \lim_{s \to t} \partial_t X_{s,t}(x) = v_{t,t}(x) = b_t(x).
\end{align}
We enforce this condition by training a neural parameterization $\hat{v}_{s,t}$ to match the interpolant's velocity along the diagonal $s=t$, yielding the standard diagonal loss:
\begin{equation}
\label{eq:flow_map_diag_loss}
    \mathcal{L}_{\text{diag}}(\hat{v})
    = \int_0^1 \mathbb{E} \left\| \hat{v}_{t,t}(I_t) - (I_1 - I_0) \right\|^2 dt.
\end{equation}
While the diagonal objective anchors the model $\hat v_{t,t}$ to the instantaneous drift $b_t$, it does not constrain the trajectory for distinct times $s \neq t$. To ensure the learned map forms a valid global trajectory, we must additionally enforce \emph{consistency constraints}. These can be expressed through three equivalent identities~\citep{boffi2025consistency}:
\begin{subequations}
\label{eq:map:rules}
\begin{align}
    \text{Semigroup:} \quad & X_{u,t}\big(X_{s,u}(x)\big) = X_{s,t}(x), \\
    \text{Lagrangian:} \quad & \partial_t X_{s,t}(x) = v_{t,t}\big(X_{s,t}(x)\big), \\
    \text{Eulerian:} \quad & \partial_s X_{s,t}(x) + v_{s,s}(x)\cdot \nabla X_{s,t}(x) = 0, 
\end{align}
\end{subequations}
for all $s, u, t \in [0,1]$. The semigroup rule enforces compositionality, ensuring that the direct transport from $s$ to $t$ is equivalent to the sequential transport through any intermediate time $u$. The Lagrangian rule dictates that the flow endpoint moves according to the instantaneous drift, while the Eulerian rule ensures invariance to the source time. 

To train the model, we employ consistency objectives that directly penalize violations of these identities. We formulate these losses as the squared residuals of the rules in~\eqref{eq:map:rules}:
\begin{align}
\label{eq:map:losses}       
    \mathcal{L}_{\text{PSD}}(\hat v) &= \iiint\limits_{0 \leq s \leq u \leq t \leq 1}\left \|\hat X_{u,t}\big(\hat X_{s,u}(x)\big) - \hat X_{s,t}(x)\right \|^2dsdudt, \\
    \mathcal{L}_{\text{LSD}}(\hat v) &= \iint\limits_{0 \leq s \leq t \leq 1} \left \|\partial_t \hat X_{s,t}(x) - \hat v_{t,t}\big(\hat X_{s,t}(x)\big)\right \|^2dsdt, \\
    \mathcal{L}_{\text{ESD}}(\hat v) &= \iint\limits_{0 \leq s \leq t \leq 1}\left\|\partial_s \hat X_{s,t}(x) + \hat v_{s,s}(x)\cdot \nabla \hat X_{s,t}(x) \right \|^2dsdt.  
\end{align} 
Our total loss is the sum of the diagonal loss and any of these consistency losses:
\begin{equation}
    \mathcal{L}_{\text{total}}(\hat v) = \mathcal{L}_{\text{diag}}(\hat v) + \mathcal{L}_{\text{cons}}(\hat v).
\end{equation}
We train the neural parametrized average velocity $\hat v$ by minimizing $\mathcal{L}_{\text{total}}$, with the minimizer yielding $\hat{v}_{s,t} = v_{s,t}$, and so $\hat X_{s,t}$ recovers the true flow map $X_{s,t}$ at optimality.

\section{Discrete Flow Maps}
\label{sec:discrete_flow_maps}
We now adapt the flow map framework to the discrete domain. For clarity, we formulate our method for distributions $p_1$ supported on the vocabulary $\mathcal{V} \subset \mathbb{R}^K$. The extension to sequences of length $L$ (i.e., distributions on $\mathcal{V}^L$) is immediate by applying these operations position-wise. Standard flow map objectives force discrete data into a Euclidean regression framework, minimizing $L^2$ errors that are geometrically ill-suited for probability distributions. In this work, we resolve this misalignment by grounding the entire flow map framework within the geometry of the probability simplex $\Delta^{K-1}$. We adopt a parametrization that naturally respects the simplex and consistency objectives based on cross-entropy and KL divergence.

\subsection{The Mean Denoiser Parametrization}

Standard flow maps parametrize the trajectory $X_{s,t}$ via the unconstrained average velocity $v_{s,t}: \mathbb{R}^K \to \mathbb{R}^K$. While effective in Euclidean space, this formulation ignores the geometry of discrete data: even if the target distribution $p_1$ is supported on the simplex, the velocity $v_{s,t}$ need not---it can take any value in $\mathbb{R}^K$. We instead seek to reparametrize the flow map in terms of an object that explicitly resides on the simplex. We achieve this by defining the flow via the \emph{mean denoiser} $\psi_{s,t}: \mathbb{R}^K \to \Delta^{K-1}$, related to the average velocity by:
\begin{equation}
\label{eq:mean_denoiser_param}
    v_{s,t}(x) = \frac{\psi_{s,t}(x) - x}{1-s}.
\end{equation}
Substituting this expression into the flow map update $X_{s,t}(x) = x + (t-s)v_{s,t}(x)$ yields the convex combination:
\begin{equation}
\label{eq:flow_map_and_psi_st}
    X_{s,t}(x) = \frac{1-t}{1-s}x + \frac{t-s}{1-s}\psi_{s,t}(x).
\end{equation}
Remarkably, $\psi_{s,t}$ is guaranteed to take values on the probability simplex. This follows from the  following characterization (see Appendix~\ref{subsec:proofs_mean_denoiser} for a proof):
\begin{propbox}
    \paragraph{Mean Denoiser.}
    \begin{restatable}{proposition*}{mean_denoiser}
    \label{prop:mean_denoiser}
    The mean denoiser $\psi_{s,t}(x)$ is the time-averaged conditional expectation of data:
    \begin{equation}
    \label{eq:mean-denoiser}
        \psi_{s,t}(x_s) = \int_s^t \, w(u)  \mathbb{E}[I_1 \mid I_u = x_u] \, du,
    \end{equation}
    where $(x_{\tau})_{\tau \in [s,t]}$ is a trajectory of the flow and $w(u) = \frac{(1-s)(1-t)}{(t-s)(1-u)^2}$ is a probability density on $[s,t]$.
    \end{restatable}
\end{propbox}
Since $\mathbb{E}[I_1 \mid I_u]$ is an expectation of one-hot vectors, it always lies on the simplex. Consequently, $\psi_{s,t}$---as a weighted convex combination of such expectations---must also reside on the simplex.

\begin{wrapfigure}[15]{r}{0.38\textwidth}
    \centering
    \vspace{-8pt}
    \includegraphics[width=0.35\textwidth]{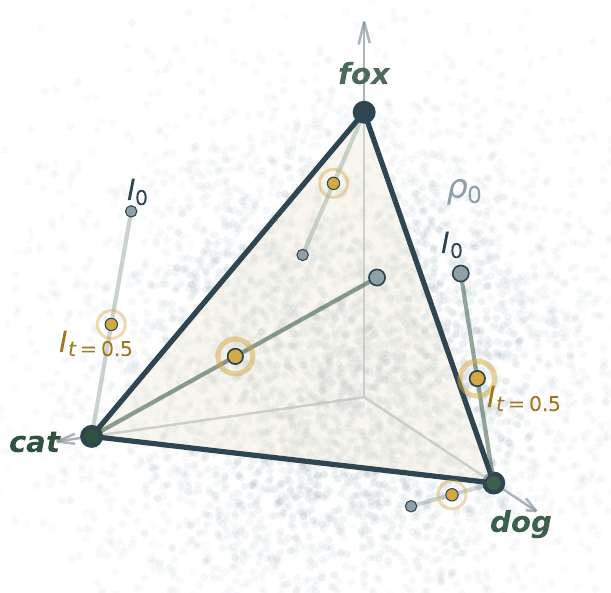}
    \vspace{-8pt}
    \caption{Interpolants with initial point $I_0 \in \mathbb R^K$ from $\rho_0$ randomly connecting to simplex vertices. }
    \label{fig:interpolant}
\end{wrapfigure}

Architecturally, we enforce this simplex constraint by parameterizing a neural network with unconstrained logits $\hat{z}_{s,t}: \mathbb{R}^K \to \mathbb{R}^K$ and defining $\hat{\psi}_{s,t}(x) = \text{Softmax}(\hat{z}_{s,t}(x))$. This ensures that the predicted target is always a valid distribution $\hat{\psi} \in \Delta^{K-1}$. We then parametrize our flow map $\hat X$ as
\begin{equation}
\label{eq:hat_flow_map_and_psi_st}
    \hat X_{s,t}(x) = \frac{1-t}{1-s}x + \frac{t-s}{1-s}\hat \psi_{s,t}(x).
\end{equation}

\paragraph{Extension to General Convex Sets.}
Although we focus on the probability simplex, this framework generalizes to any convex set. If the data distribution $p_{1}$ lies in $\Lambda\subset \mathbb{R}^K$, then the mean denoiser $\psi_{s,t}\in \text{ConvexHull}(\Lambda)$. In our case, $\Lambda=\{e_1,\cdots,e_{K}\}$ and $\text{ConvexHull}(\Lambda)=\Delta^{K-1}$. One can replace the Softmax with a suitable link function to ensure $\hat{\psi}_{s,t} \in \text{ConvexHull}(\Lambda)$ by design.

\subsection{Training Objectives}
We now turn to training $\hat{\psi}_{s,t}$. Since $\hat{\psi}_{s,t}$ is constrained to output valid probability distributions, the usual flow map training objectives—originally formulated as Euclidean regression losses—can be reformulated to act on distributions. This allows us to use cross-entropy and KL-divergence losses that respect the geometry of the simplex. As in standard flow map training, we employ two complementary objectives: a diagonal loss that anchors $\hat{\psi}_{t,t}$, and a consistency loss that enforces geometric validity of $\hat{\psi}_{s,t}$ for $s < t$.

\subsubsection{Diagonal Loss}
We begin with the following identity which states that the diagonal of the mean denoiser is the standard denoiser:
\begin{propbox}
    \paragraph{Diagonal Identity for $\psi_{s,t}$.}
    \begin{restatable}{proposition*}{diagonal_loss}
    \label{prop:diagonal}
    For any $t \in [0,1]$, the mean predictor satisfies:
    \begin{equation}
    \label{eq:md_diagonal}
        \psi_{t,t}(x) = \mathbb{E}[I_1 \mid I_t = x].
    \end{equation}
    \end{restatable}
\end{propbox}
Based on this identity, we train $\hat{\psi}_{t,t}$ to predict the target class $I_1$ given the noisy state $I_t$. Since $I_1$ takes values in $\mathcal{V}$, we can minimize the expected cross-entropy loss:
\begin{equation}
    \mathcal{L}_{\text{diag}}(\hat{\psi}) = \int_0^1 \mathbb{E}\left[ -\sum_{k=1}^K I_1^{(k)} \log \hat{\psi}_{t,t}^{(k)}(I_t) \right]dt,
\end{equation}
where the expectation is over the joint distribution of $I_t$ and $I_1$. Any proper scoring rule or Bregman divergence could alternatively be used here. Minimizing this loss ensures that $\hat \psi_{t,t} = \psi_{t,t}$ at optimality.

\subsubsection{Consistency Loss}
To learn a valid flow map $X_{s,t}$, the model must satisfy the consistency constraints in~\eqref{eq:map:rules}. We show here that we can rewrite these fundamental flow identities in terms of the mean denoiser $\psi_{s,t}$.
\begin{propbox}
    \paragraph{Discrete Flow Map Identities.}
    \begin{restatable}{proposition*}{identities}
    \label{prop:identities}
    For any  $0\le s < u < t \le 1$, each of the following identities, in conjunction with~\eqref{eq:md_diagonal}, characterizes $\psi_{s,t}$ uniquely:
    \begin{align}
        \small\text{\textbf{Semigroup:}}\normalsize & \qquad \psi_{s,t}(x) = \alpha_{s,u,t} \psi_{s,u}(x) + \beta_{s,u,t} \psi_{u,t}(X_{s,u}(x)), \\
        \small\text{\textbf{Lagrangian:}}\normalsize & \qquad \psi_{s,t}(x) = \psi_{t,t}(X_{s,t}(x)) - \gamma_{s,t} \partial_t \psi_{s,t}(x),\label{eq:md_lagrangian} \\
        \small\text{\textbf{Eulerian:}}\normalsize & \qquad \partial_s \psi_{s,t}(x) + J_x \psi_{s,t}(x) b_s(x) = \kappa_{s,t} (\psi_{s,t}(x) - \psi_{s,s}(x)), \label{eq:md_eulerian}
    \end{align}
    Here, $J_x$ is the Jacobian with respect to $x$ and we define the coefficients $\alpha_{s,u,t} = \tfrac{(u-s)(1-t)}{(t-s)(1-u)}\in(0,1)$ and $\beta_{s,u,t} = \tfrac{(t-u)(1-s)}{(t-s)(1-u)}\in(0,1)$  that sum to $1$, as well as  $\gamma_{s,t} = \tfrac{(t-s)(1-t)}{1-s}\in(0,1)$ and $\kappa_{s,t} = \tfrac{1-t}{(1-s)(t-s)}>0$.
    \end{restatable}
\end{propbox}
We now convert these geometric identities into tractable training objectives. Our goal is to leverage the discrete geometry of the model to formulate exact losses.

\paragraph{Consistency via Semigroup Loss.}
We would like to enforce the semigroup identity directly on our network $\hat{\psi}$:
\begin{equation}
\label{eq:psd_on_hat_psi}
    \hat{\psi}_{s,t}(x) = \alpha_{s,u,t} \hat{\psi}_{s,u}(x) + \beta_{s,u,t} \hat{\psi}_{u,t}(\hat{X}_{s,u}(x)).
\end{equation}
Since the network output is constrained to the simplex, the right-hand side---a convex combination of probability vectors since $\alpha_{s,u,t},\beta_{s,u,t}\in(0,1)$ and $\alpha_{s,u,t} + \beta_{s,u,t}=1$---defines a valid target distribution. Treating this composite prediction as the teacher, we distill it into the student $\hat{\psi}_{s,t}$ by minimizing the KL divergence:
\begin{equation}
    \mathcal{L}_{\text{PSD}}(\hat{\psi}) = \mathbb{E} \left[ D_{\text{KL}}\left( \text{sg}\left[ \alpha_{s,u,t} \hat{\psi}_{s,u}(I_s) + \beta_{s,u,t} \hat{\psi}_{u,t}(\hat{X}_{s,u}(I_s)) \right] \parallel \hat{\psi}_{s,t}(I_s) \right) \right],
\end{equation}
where $\text{sg}[\cdot]$ denotes the stop-gradient operator and where the expectation is over the law of $I_s$ and a distribution of $s,u,t$ such that $s\leq u\leq t$ with full support. Minimizing this loss enforces~\eqref{eq:psd_on_hat_psi}, which, together with the diagonal condition, ensures that $\hat{X}$ recovers the true target flow map. We note that any divergence $D$ could be used above (including reverse KL instead of forward KL).

\paragraph{Consistency via Lagrangian Loss.}
Alternatively, we can choose a consistency loss that ensures the Lagrangian identity~\eqref{eq:md_lagrangian} holds for our network $\hat \psi_{s,t}$:
\begin{equation}
\label{eq:hat_md_lagrangian}
    \hat \psi_{s,t}(x) = \hat \psi_{t,t}(\hat X_{s,t}(x)) - \gamma_{s,t} \,\partial_t \hat \psi_{s,t}(x).
\end{equation}
Since the true mean denoiser $\psi_{s,t}$ takes values on the simplex, both sides of~\eqref{eq:md_lagrangian} are probability distributions. However, when enforcing this constraint on our neural network $\hat{\psi}_{s,t}$, the right-hand side of~\eqref{eq:hat_md_lagrangian} is not guaranteed to lie on the simplex, since the time derivative $\partial_t \hat \psi_{s,t}$ can push it off. To rigorously enforce the Lagrangian identity, we instead derive an equivalent condition in \emph{logit space} and then recast it to probability distributions.
\begin{propbox}
    \paragraph{Lagrangian Logit Consistency.}
    \begin{restatable}{proposition*}{lagrangianlogit}
    \label{prop:lagrangian_logit}
    Let $\psi_{s,t}(x) = \text{Softmax}(z_{s,t}(x))$ for any logit lift $z_{s,t} : \mathbb{R^K} \to \mathbb{R}^K$. Define the Lagrangian teacher $T^{\text{LSD}}$, operating on mean denoisers, by:
   \begin{equation}
        T_{s,t}^{\mathrm{LSD}}(\psi)(x)
        := \mathrm{Softmax}\!\left(
        z_{t,t}(X_{s,t}(x))
        -\log\!\left(\mathbf{1}+\gamma_{s,t}\Big(\partial_t z_{s,t}(x)-\langle \psi_{s,t}(x),\partial_t z_{s,t}(x)\rangle\mathbf{1}\Big)\right)
        \right).
\end{equation}
    By the shift-invariance of $\mathrm{Softmax}$, $T^{\mathrm{LSD}}(\psi)$ is independent of the chosen lift $z$. Then the Lagrangian identity~\eqref{eq:md_lagrangian} is equivalent to:
    \begin{equation}
    \label{eq:logit_lagrangian}
        \psi_{s,t}(x) = T_{s,t}^{\text{LSD}}(\psi)(x).
    \end{equation}
    \end{restatable}
\end{propbox}
The Lagrangian teacher $T^{\text{LSD}}$ always outputs a probability distribution and therefore defines a geometrically valid target for training. We minimize the forward KL divergence from the target to the student $\hat{\psi}_{s,t}$:
\begin{equation} 
    \mathcal L_{\text{LSD}}(\hat \psi) =  \mathbb E\left[ D_{\text{KL}}\left( \text{sg}[T^{\text{LSD}}_{s,t}(\hat \psi)(I_s)] \parallel \hat \psi_{s,t}(I_s) \right) \right],
\end{equation}
where the expectation is taken over $\text{Law}(I_s)$ and a distribution over $s,t$ such that $s\leq t$. Minimizing this loss ensures that both Lagrangian consistency rules~\eqref{eq:md_lagrangian} and~\eqref{eq:logit_lagrangian} are satisfied by $\hat \psi_{s,t}$.

\paragraph{Consistency via Eulerian Loss.}
Next, we consider the Eulerian perspective, and enforce~\eqref{eq:md_eulerian} on $\hat \psi_{s,t}$:
\begin{equation}
    \partial_s \hat \psi_{s,t}(x) + J_x \hat \psi_{s,t}(x) b_s(x) = \kappa_{s,t} (\hat \psi_{s,t}(x) - \hat \psi_{s,s}(x)).
\end{equation}
As with the Lagrangian case, we can derive an equivalent condition in logit space.
\begin{propbox}
    \paragraph{Eulerian Logit Consistency.}
    \begin{restatable}{proposition*}{eulerianlogit}
    \label{prop:eulerian_logit}
    Let $\psi_{s,t}(x)=\mathrm{Softmax}(z_{s,t}(x))$ for any logit lift $z_{s,t}:\mathbb{R}^K\to\mathbb{R}^K$. Define the Eulerian teacher $T^{\mathrm{ESD}}$, operating on mean denoisers, by
    \begin{equation}
    T_{s,t}^{\mathrm{ESD}}(\psi)(x)
    := \mathrm{Softmax}\!\left(
    z_{s,s}(x)
    -\log\!\left(\mathbf{1}-\kappa_{s,t}^{-1}\Big(D_s z_{s,t}(x)-\langle \psi_{s,t}(x),D_s z_{s,t}(x)\rangle\mathbf{1}\Big)\right)
    \right).
    \end{equation}
    Here $D_s z_{s,t}=\partial_s z_{s,t}+J_x z_{s,t}\, b_s$ is the total derivative along $b_s$. Then the Eulerian identity is equivalent to
    \begin{equation}
    \label{eq:logit_eulerian}
    \psi_{s,t}(x)=T_{s,t}^{\mathrm{ESD}}(\psi)(x).
    \end{equation}
    \end{restatable}
\end{propbox}
We minimize the forward KL divergence from the target $T^{\text{ESD}}$ to the student $\hat{\psi}_{s,t}$:
\begin{equation} 
    \mathcal L_{\text{ESD}}(\hat \psi) =  \mathbb E\left[ D_{\text{KL}}\left( \text{sg}[T^{\text{ESD}}_{s,t}(\hat \psi)(I_s)] \parallel \hat \psi_{s,t}(I_s) \right) \right],
\end{equation}
where the expectation is taken over $\text{Law}(I_s)$ and a distribution over $s,t$ such that $s\leq t$. Minimizing this loss ensures that both Eulerian consistency rules~\eqref{eq:md_eulerian} and~\eqref{eq:logit_eulerian} are satisfied by $\hat \psi_{s,t}$.

\section{Algorithmic Details}
We now discuss the main practical choices used in our implementation. We first consider the choice of interpolants and time schedules, including reparameterizations that distribute denoising progress more evenly over time. We then describe conditional variants of the model and the associated guidance mechanisms, before turning to the stabilized logit-space objectives and loss weightings.
\subsection{Interpolants and Schedules}

\paragraph{Time Reparameterization.}
\label{sec:time_reparametrisation}
A useful algorithmic degree of freedom is to reparameterize time so that denoising progress is distributed more evenly along the trajectory. Instead of the linear interpolant $I_t = (1-t) I_0 + t I_1$, we use
\begin{equation}
    I_t = (1-\beta(t)) I_0 + \beta(t) I_1,
\end{equation}
where $\beta : [0,1] \to [0,1]$ is increasing with $\beta(0)=0$ and $\beta(1)=1$. This does not change the endpoints or the underlying noise-to-data path; it only changes how quickly that path is traversed. Concretely, we choose $\beta$ so that $\mathbb{P}(\arg\max(I_t)=\arg\max(I_1))$ increases approximately linearly in the reparameterized time, following the related time-reparameterization ideas of \citet{pynadath2025candihybriddiscretecontinuousdiffusion,sahoo2025diffusion}. In this way, equal time increments correspond to equal gains in identifying the final token, rather than having most decisions concentrated in a narrow part of the trajectory.

When training with the reparameterized schedule $\beta(t)$, it is often preferable to condition the network on $\beta(s)$ and $\beta(t)$ rather than the raw times $s$ and $t$. Using $\beta(t)$ instead of the raw time $t$ can make the network easier to train, since it avoids forcing the model to represent a sharp dependence on $t$. It also makes it easier to use a different reparameterization at sampling or distillation time, which we found useful in practice.

\paragraph{General Interpolants.}
Although we focus on the linear interpolant in the main text, the construction extends directly to a broader class of interpolants. In particular, one may replace $I_t=(1-t)I_0+tI_1$ with a general interpolant of the form 
\begin{equation}
    I_t=\alpha_t I_0+\beta_t I_1,
\end{equation}
with suitable endpoint conditions, and all of the main objects, including the mean denoiser, flow map parameterization, and consistency identities, admit corresponding schedule-dependent forms. We defer the full development to Appendix~\ref{sec:proofs}, where these general formulas are derived and shown to recover the linear case as a special instance.

\paragraph{Position-Dependent Schedules.} A natural extension is to replace the shared interpolant with \emph{position-dependent schedules}. Instead of noising every token at the same rate, we can define for each position $\ell \in \{1,\dots,L\}$ an interpolant
\begin{equation}
    I_t^\ell = (1 - \beta_t)^{(\ell)} I_0^\ell + \beta_t^{(\ell)} I_1^\ell,
\end{equation}
with the schedules ordered so that earlier positions are revealed sooner than later ones, e.g. $\beta_t^{(1)} \ge \beta_t^{(2)} \ge \cdots \ge \beta_t^{(L)}$ for all $t$. In this construction, the model sees a comparatively cleaner prefix and a noisier suffix at intermediate times, introducing an autoregressive bias into an otherwise parallel model. 

\subsection{Conditional Generation and Guidance}
\begin{figure}[t]
    \centering
    \includegraphics[width=0.95\linewidth]{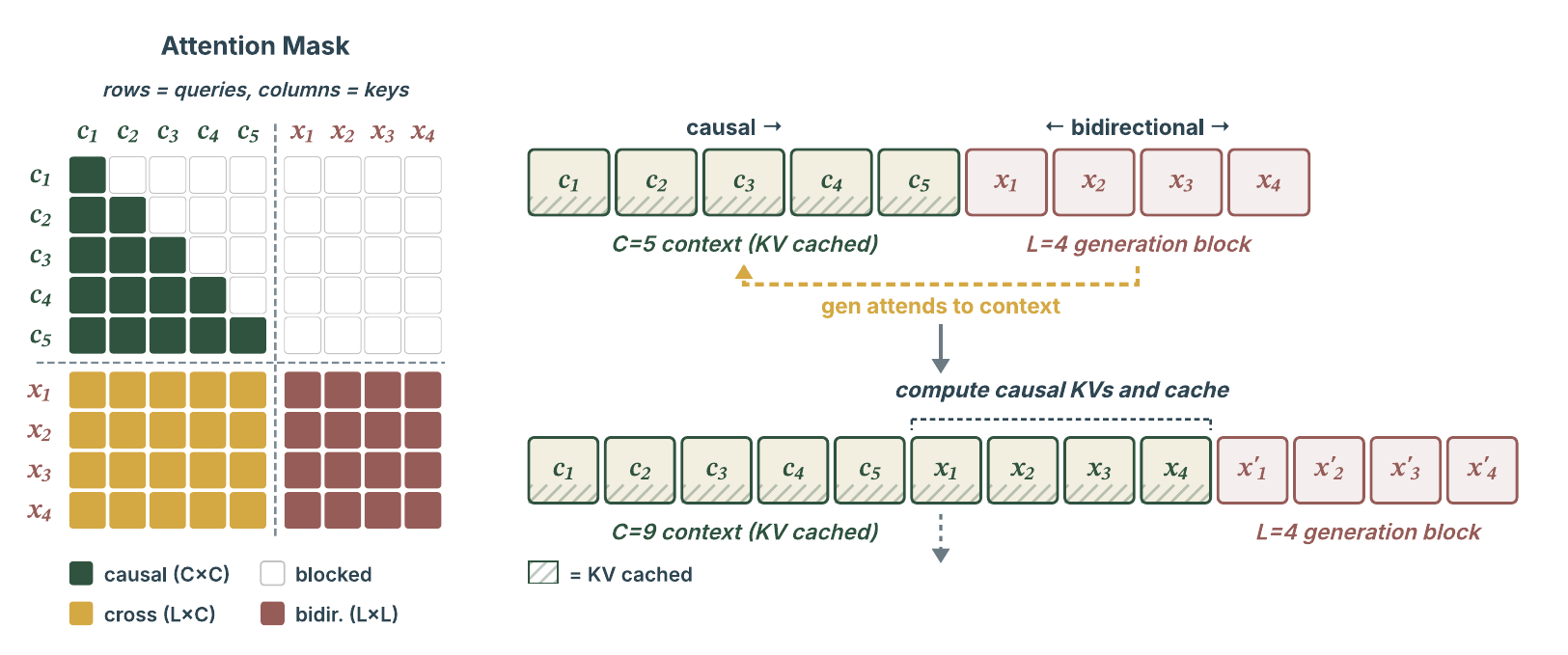}
    \caption{Illustration of block generation with mixed attention masking. A fixed context (green) is processed with causal attention, while a block of future tokens (red) is generated in parallel using bidirectional attention within the block and cross-attention to the context. The attention mask (left) visualizes the masking structure, and the right panel shows how cached key–value (KV) states enable efficient reuse of context while generating successive blocks.}
    \label{fig:block}
\end{figure}
A natural extension of the framework is to train the model conditionally on a prompt or prefix. Given a variable-length context $\mathbf{c}$ and continuation $\mathbf{x} \in \mathcal{V}^L$, we keep the context fixed and apply the flow only to the continuation, thereby modeling the conditional distribution $p_1(\mathbf{x}\mid \mathbf{c})$. In this setting, the model learns to generate a suffix given a clean prefix, rather than producing the entire sequence from scratch.

\paragraph{Block Generation.}
This conditional formulation naturally enables \emph{block generation}. Given a context $\mathbf{c}$, the model can generate an entire block of $L$ future tokens in parallel, append the generated block to the context, and then repeat the process to generate subsequent blocks. In this way, long sequences can be produced through a sequence of parallel block updates rather than fully autoregressive token-by-token decoding.

\paragraph{Classifier-Free and Model Guidance.}
\label{sec:cfg}
The same conditional setup also makes it natural to incorporate \emph{classifier-free guidance} (CFG)~\citep{ho2022classifierfreediffusionguidance}. Concretely, one trains both a conditional model with context $\mathbf{c}$ and an unconditional model in which the context is dropped. For the instantaneous denoiser, this guidance can be expressed directly at the level of the drift. Writing the conditional and unconditional drifts as $b_t(\mathbf{x};\mathbf{c})$ and $b_t(\mathbf{x})$, respectively, we define the guided drift by
\begin{equation}
b_t^{\mathrm{CFG}}(\mathbf{x};\mathbf{c})
:=
b_t(\mathbf{x})
+
\omega \bigl(b_t(\mathbf{x};\mathbf{c}) - b_t(\mathbf{x})\bigr),
\end{equation}
where $\omega \ge 0$ controls the guidance strength. For the linear interpolant, since $b_t(x)=\frac{\psi_{t,t}(x)-x}{1-t}$, this is equivalently viewed as applying CFG directly to the instantaneous denoiser. Accordingly, the same conditional guidance mechanism can then be inherited by the distilled flow map, enabling guided generation at test time in the one-step or few-step setting, commonly referred to as \emph{model guidance}~\citep{tang2025diffusionmodelsclassifierfreeguidance}.

\paragraph{Support Preservation under Guidance.} A key question is whether applying CFG with drift $b_t^{\mathrm{CFG}}$ still yields terminal samples that land on vertices of the simplex, as required for the final outputs to correspond to valid tokens. In fact, this is true: Theorem~3 of \cite{azangulov2026adaptivediffusionguidancestochastic} shows that guided sampling recovers the support of the conditional data distribution. In our setting, since that support is contained in the simplex vertices, it follows that guided generation still terminates at valid discrete tokens.

\subsection{Loss Implementation}

\paragraph{Stable Logit-Space Targets.}
For the logit-space LSD and ESD objectives, a direct implementation of the teacher can be numerically unstable because the correction coefficients become ill-conditioned near the boundary. In particular, for ESD under the linear interpolant, the teacher takes the form
\begin{equation}
    T^{\mathrm{ESD}}_{s,t}(\psi)(x)
    =
    \mathrm{Softmax}\!\left(
        z_{s,s}(x)
        -
        \log\!\left(
            \mathbf{1}
            -
            \kappa_{s,t}^{-1}
            \delta_{s,t}(x)
        \right)
    \right),
\end{equation}
where
\begin{equation}
    \delta_{s,t}(x)
    :=
    D_s z_{s,t}(x)
    -
    \langle \psi_{s,t}(x), D_s z_{s,t}(x)\rangle \mathbf{1},
    \qquad
    \kappa_{s,t}
    =
    \frac{1-t}{(1-s)(t-s)}.
\end{equation}
As $t \to 1$, the factor $\kappa_{s,t}^{-1}=\frac{(1-s)(t-s)}{1-t}$ becomes numerically large, and the analogous coefficient in the LSD correction has the same issue. To avoid forming these unstable ratios explicitly, we rewrite
\begin{equation}
    \mathbf{1}-\kappa_{s,t}^{-1}\delta_{s,t}(x)
    =
    \frac{(1-t)\mathbf{1}-(1-s)(t-s)\delta_{s,t}(x)}{1-t}.
\end{equation}
Since $\mathrm{Softmax}$ is invariant to shifts of the logits by a scalar multiple of $\mathbf{1}$, the scalar denominator contributes only an additive constant in log-space and therefore cancels. This yields the equivalent ESD target
\begin{equation}
T^{\mathrm{ESD}}_{s,t}(\psi)(x) = \mathrm{Softmax}\!\left(
            z_{s,s}(x)
            -
            \log\!\left(
                (1-t)\mathbf{1}-(1-s)(t-s)\delta_{s,t}(x)
            \right)
        \right).
\end{equation}
In practice, this rearrangement avoids the poor conditioning of the correction coefficients near the boundary while leaving the target distribution unchanged; LSD is handled analogously.

\paragraph{Loss Weighting.}
We apply a detached weighting to the loss in order to stabilize optimization. Let $q(x)=\hat{\psi}_{s,t}(x)\in\Delta^{K-1}$ denote the student prediction, and let $p(x)\in\Delta^{K-1}$ denote the teacher target, for example $p(x)=\mathrm{sg}[T_{s,t}(\hat{\psi})(x)]$. We write the simplex mismatch as
\begin{equation}
    \Delta_{s,t}(x) := q(x)-p(x).
\end{equation}
Since the gradient of the forward KL loss $D_{\mathrm{KL}}(p(x)\|q(x))$ with respect to the student logits is exactly $\Delta_{s,t}(x)$, we can control its magnitude by multiplying the loss by a detached scalar weight depending only on $\|\Delta_{s,t}(x)\|_2$. Concretely, we use
\begin{equation}
\label{eq:loss_adaptive}
    w_{s,t}(x)
    :=
    \mathrm{sg}\!\left[
        \bigl(\|\Delta_{s,t}(x)\|_2^2 + c\bigr)^{-r}
    \right],
\end{equation}
with, for example, $c=10^{-6}$ and $r=0.5$, and optimize the weighted objective
\begin{equation}
    \mathcal{L}_{\mathrm{wKL}}
    =
    \mathbb{E}\!\left[
        w_{s,t}(x)\,
        D_{\mathrm{KL}}\bigl(p(x) \,\|\, q(x)\bigr)
    \right].
\end{equation}
Because the weight is detached, it simply rescales the student logit gradient, yielding a more robust KL distillation objective that downweights overly large updates for $r>0$ without changing the optimum.

\section{Experiments}
We evaluate DFMs on the One Billion Word (LM1B) \citep{chelba2014billionwordbenchmarkmeasuring} and OpenWebText (OWT) \citep{Gokaslan2019OpenWeb} datasets. We tokenize LM1B with \textit{bert-base-uncased} and use a sequence length of 128, while for OWT we use the \textit{GPT-2} tokenizer and a sequence length of 1024. 

\paragraph{Training.} In principle, one can either train the flow matching model (i.e. the diagonal) first and then distill using one of the consistency losses, or train the diagonal and off-diagonal jointly via self-distillation. In our experiments, we adopt the former approach. For both datasets, we train the diagonal for 1M steps, and then train the off-diagonal for an additional 200k steps on LM1B and 100k steps on OWT, using both PSD and ESD consistency variants. Across both datasets and training stages, we use a batch size of 512 and the Adam optimizer \citep{kingma2017adam} with a learning rate of $3\times10^{-4}$. For all experiments, we use a 170M parameter diffusion transformer~\citep{peebles2023scalable} with 12 blocks, rotary positional embeddings, and adaptive layer normalization for timestep conditioning, following recent work \citep{sahoo2024simple}. Additional hyperparameter and experimental details are provided in Appendix~\ref{sec:experimental_details}.

\paragraph{Results.}

We compare DFMs against recent accelerated discrete diffusion baselines: Duo with DCD~\citep{sahoo2025diffusion}, MDLM with SDTT~\citep{deschenaux2025autoregressionfastllmsselfdistillation}, and both methods combined with Di4C~\citep{hayakawa2025distillation}. We also compare against concurrent works on Categorical Flow Maps~\citep{roos2026categorical} and Flow Map Language Models (FMLMs)~\citep{lee2026flowmaplanguagemodels}, using the reported results in the latter for all relevant baselines. Following the standard evaluation protocol for non-autoregressive language models, we report generative perplexity (gen. PPL) computed using \textit{GPT-2 Large}~\citep{radford2019language}, together with unigram entropy. Across both datasets, DFMs outperform all baselines in the few-step regime we consider in terms of generative perplexity, while generally preserving diversity. Among the DFM variants, ESD outperforms PSD at 2 and 4 NFEs on both entropy and generative perplexity, although it exhibits mode collapse at 1 NFE. Example generations from DFMs on LM1B are shown in Figure~\ref{fig:samples}.

\begin{table*}[t]
\centering
\footnotesize
\setlength{\tabcolsep}{5pt}
\renewcommand{\arraystretch}{1.15}

\begin{tabular}{llccc|llccc}
\toprule
\multicolumn{5}{c}{\textbf{LM1B}} & \multicolumn{5}{c}{\textbf{OWT}} \\
\cmidrule(lr){1-5} \cmidrule(lr){6-10}
Method & Metric & 1 & 2 & 4 & Method & Metric & 1 & 2 & 4 \\
\midrule

\multirow{2}{*}{Duo + DCD}
& gen.\ PPL $\downarrow$ & 1224.52 & 520.08 & 210.88
& \multirow{2}{*}{Duo + DCD}
& gen.\ PPL $\downarrow$ & 5743.29  & 891.16  & 250.86 \\
& entropy $\uparrow$ & 4.33 & 4.20 & 4.23
& & entropy $\uparrow$& 6.02 & 5.41 & 5.37 \\

\multirow{2}{*}{Duo + Di4C}
& gen.\ PPL & 292.94 & 247.69 & 150.67
& \multirow{2}{*}{Duo + Di4C}
& gen.\ PPL & 370.51  & 210.22 & 154.67 \\
& entropy & 3.79 & 3.87 & 4.00
& & entropy & 3.92 & 4.63 & 4.85 \\

\multirow{2}{*}{MDLM + SDTT}
& gen.\ PPL & 1429.48 & 602.14 & 241.01
& \multirow{2}{*}{MDLM + SDTT}
& gen.\ PPL & 1260.86 & 877.22 & 339.73 \\
& entropy & 4.31 & 4.28 & 4.28
& & entropy & 5.26 & 5.34 & 5.38 \\

\multirow{2}{*}{MDLM + Di4C}
& gen.\ PPL & 1217.10 & 621.59 & 247.32
& \multirow{2}{*}{MDLM + Di4C}
& gen.\ PPL & 1298.80 & 758.23 & 239.27 \\
& entropy & 4.38 & 4.37 & 4.00
& & entropy & 5.29 & 5.35 & 5.40 \\

\multirow{2}{*}{CFM}
& gen.\ PPL & 269.72 & 267.39 & 267.97
& \multirow{2}{*}{CFM}
& gen.\ PPL & -- & -- & -- \\
& entropy & 3.10 & 3.15 & 3.28
& & entropy & -- & -- & -- \\

\multirow{2}{*}{FMLM}
& gen.\ PPL & 119.34 & 110.19 & 98.76
& \multirow{2}{*}{FMLM}
& gen.\ PPL & 168.30 & 133.29 & 111.31 \\
& entropy & 4.16 & 4.21 & 4.21
& & entropy & 5.17 & 5.25 & 5.26 \\

\arrayrulecolor{gray!40}\cmidrule(lr){1-5}\cmidrule(lr){6-10}
\arrayrulecolor{black}

\multirow{2}{*}{\textbf{DFM (PSD)}}
& gen.\ PPL & 94.08 & 87.42 & 78.89
& \multirow{2}{*}{\textbf{DFM (PSD)}}
& gen.\ PPL & 180.29 & 152.83 & 122.32 \\
& entropy & 4.06 & 4.08 & 4.10
& & entropy & 4.91 & 5.03 & 5.10 \\

\multirow{2}{*}{\textbf{DFM (ESD)}}
& gen.\ PPL & 68.11 & 77.60 & 71.53
& \multirow{2}{*}{\textbf{DFM (ESD)}}
& gen.\ PPL & 5.33 & 108.91 & 77.08 \\
& entropy & 3.79 & 4.11 & 4.13
& & entropy & 0.26 & 5.15 & 5.27 \\

\bottomrule
\end{tabular}

\caption{Generative perplexity ($\downarrow$) and entropy ($\uparrow$) across number of function evaluations (NFEs) for LM1B and OWT.}
\label{tab:method-step-metrics}
\end{table*}

In Table~\ref{tab:nfe_combined}, we present performance over a larger range of NFEs for our models, including performance after only training the diagonal, before consistency distillation. This table directly illustrates the effectiveness of consistency training using the PSD and ESD losses, with the resulting few-step sampler significantly outperforming the model obtained from training the diagonal alone. 

\paragraph{Classifier-Free Guidance (CFG).}

\begin{wraptable}[10]{r}{0.43\textwidth}
    \vspace{-1.5\baselineskip}
    \centering
    \small
    \setlength{\tabcolsep}{4pt}
    \renewcommand{\arraystretch}{1.05}
    \begin{tabular}{lcc}
        \toprule
        $\omega$ & gen.\ PPL $\downarrow$ & entropy $\uparrow$ \\
        \midrule
        0.0 & 56.31 & 5.20 \\
        0.5 & 46.78 & 5.06 \\
        1.0 & 36.44 & 4.94 \\
        1.5 & 33.22 & 4.87 \\
        2.0 & 30.98 & 4.81 \\
        \bottomrule
    \end{tabular}
    \caption{gen.\ PPL ($\downarrow$) and entropy ($\uparrow$) across CFG scales, $\omega$. Samples are generated in four blocks of 256 tokens, each with 1024 steps.}
    \label{tab:cfg_metrics}
    \vspace{-0.8\baselineskip}
\end{wraptable}

We next study block-conditional generation on OWT. Following Section~\ref{sec:cfg}, we train the model to generate blocks of 256 tokens conditioned on prompts of previous tokens. At sampling time, this also enables classifier-free guidance (CFG). In this experiment, we consider only the many-step flow model, though it can be readily distilled into a few-step generator. We report results in Table~\ref{tab:cfg_metrics} for guidance scales $\omega \in \{0.0, 0.5, 1.0, 1.5, 2.0\}$, where $\omega=0$ corresponds to unconditional generation and $\omega=1$ to standard conditional generation.

As $\omega$ increases beyond standard conditional generation ($\omega=1$), both generative perplexity and entropy decrease. This is consistent with observations about CFG in continuous domains, such as image generation, where stronger guidance typically improves sample fidelity at the cost of diversity~\citep{ho2022classifierfreediffusionguidance}. Example generations for a range of guidance strengths are presented in Appendix~\ref{sec:example_generations}.

\begin{table*}[t]
\centering
\small
\setlength{\tabcolsep}{5pt}
\renewcommand{\arraystretch}{1.15}
\begin{tabular}{lllcccccccc}
\toprule
Dataset & Stage & Metric & 1 & 2 & 4 & 8 & 16 & 128 & 256 & 1024 \\
\midrule

\multirow{6}{*}{LM1B}
& \multirow{2}{*}{Diagonal}
& gen.\ PPL $\downarrow$ & 2.05 & 69.59 & 204.13 & 124.64 & 101.89 & 75.82 & 72.75 & 68.33 \\
& & entropy $\uparrow$ & 0.76 & 2.87 & 4.37 & 4.32 & 4.27 & 4.19 & 4.17 & 4.16 \\

& \multirow{2}{*}{+ Distillation (PSD)}
& gen.\ PPL & 94.08 & 87.42 & 78.89 & 69.90 & 64.90 & 58.38 & 56.59 & 56.31 \\
& & entropy & 4.06 & 4.08 & 4.10 & 4.10 & 4.11 & 4.10 & 4.10 & 4.11 \\

& \multirow{2}{*}{+ Distillation (ESD)}
& gen.\ PPL & 68.11 & 77.60 & 71.53 & 65.61 & 59.92 & 55.70 & 56.88 & 58.27 \\
& & entropy & 3.79 & 4.11 & 4.13 & 4.13 & 4.13 & 4.11 & 4.12 & 4.12 \\

\midrule

\multirow{6}{*}{OWT}
& \multirow{2}{*}{Diagonal}
& gen.\ PPL & 29.79 & 9.90 & 56.03 & 180.79 & 122.20 & 61.92 & 55.63 & 47.07 \\
& & entropy & 1.55 & 0.91 & 2.84 & 5.20 & 5.52 & 5.28 & 5.22 & 5.12 \\

& \multirow{2}{*}{+ Distillation (PSD)}
& gen.\ PPL & 180.29 & 152.83 & 122.32 & 98.54 & 82.51 & 56.00 & 51.81 & 47.82 \\
& & entropy & 4.91 & 5.03 & 5.10 & 5.11 & 5.09 & 5.00 & 4.97 & 4.97 \\

& \multirow{2}{*}{+ Distillation (ESD)}
& gen.\ PPL & 5.33 & 108.91 & 77.08 & 62.98 & 55.03 & 41.90 & 39.08 & 36.48 \\
& & entropy & 0.26 & 5.15 & 5.27 & 5.23 & 5.18 & 5.04 & 5.00 & 4.95 \\

\bottomrule
\end{tabular}
\caption{Generative perplexity ($\downarrow$) and entropy ($\uparrow$) across number of function evaluations (NFEs).}
\label{tab:nfe_combined}
\end{table*}

\section{Related Work}
\paragraph{Language Models via Continuous Flow and Diffusion.}Various works have explored the use of continuous flow and diffusion models for language generation, with the dominant difference being how discrete data is represented in a continuous space. One approach is to represent language data via learned word or token embeddings  \citep{dieleman2022continuous, li2022diffusion}. \citet{li2022diffusion} introduce \emph{Diffusion-LM}, a diffusion model for language modeling in continuous space.  Word embeddings are  trained jointly with the diffusion model. Guidance via continuous diffusion is shown to enable controllable text generation. Another approach is to represent discrete data via one-hot vectors in continuous space. The support of the data distribution is then constrained to a finite (measure zero) subset of $\mathbb{R}^d$, a fact that can be leveraged in  ways similar to ours (see \cref{eq:mean-denoiser}). For example, \emph{Variational Flow Matching} \citep{eijkelboom2024variational} reframes flow matching in this setting as minimizing a classifier via cross-entropy. Other works represent discrete data as lying on the simplex or other finite-dimensional manifolds, e.g. allowing to construct flows that remain on the simplex  or finite-dimensional manifolds  (not just at time $t=1$ but also for $t<1$)\citep{stark2024dirichlet, davis2024fisher}. Concurrent with our work, \cite{lee2026flowmaplanguagemodels} study a closely related discrete flow-map approach for language modeling. While the two works share similar diagonal training, their off-diagonal treatment focuses on the semigroup/PSD formulation, whereas we additionally develop and emphasize the Eulerian perspective.

Recent work on Categorical Flow Maps \citep{roos2026categorical} has sought to adapt flow maps to discrete data. However, it does not fully exploit the geometric structure of the probability simplex, instead relying on composite loss bounds or inexact objectives. By contrast, our approach places the simplex geometry and exact cross-entropy and KL divergence losses at the center of the framework, which we show leads to improved empirical performance.

\paragraph{Language Models via Discrete Diffusion.} Language models via discrete diffusion models have recently attracted significant attention \citep{austin2021structured, campbell2022continuous,lou2024discrete, sahoo2024simple, shi2024simplified,gat2024discrete, shaul2024flow}. At scale, they have been shown to lead to significant speed-ups  \citep{nie2025large}. Despite the name ``discrete diffusion'', these models are based on discrete-time or continuous-time Markov chains defined on discrete state spaces. While distillation has been explored  \citep{sahoo2025diffusion}, performance of distilled discrete diffusion models remains limited. The underlying limitation is that discrete diffusion models  update each token independently per step. While this functional form is correct for small steps, it is not for large steps amortized in distilled models. Therefore, the limited expressivity of the updates in discrete diffusion models naturally leads to performance degradation. In contrast, our approach here focuses on continuous flow maps that do not suffer from such limited expressivity.

\paragraph{Flow Maps.} Our work is closely related to the  flow maps framework  \citep{boffi2024flow, boffi2025consistency}, whose various loss functions we translate here to modeling of discrete data in continuous spaces. Mean Flows are a reparameterization of flow maps in the average velocity parameterization recently demonstrating state-of-the-art performance \citep{geng2025mean, geng2025improved}. We introduce an equivalent concept of the mean flow here, the mean denoiser. As demonstrated in this work,  the mean denoiser is a simple reparameterization of the flow map naturally suited for modeling discrete data as it preserves convex constraints of the data. Various other works have demonstrated impressive empirical scaling of flow maps in the image and video domain \citep{sabour2025align, zhou2025terminal}, highlighting the potential of scaling up flow map distillation for language further.

\section*{Acknowledgments}
We would like to thank Philippe Rigollet, Grant Rotskoff, Oscar Davis and Nick Boffi for helpful discussions. PP is supported by the EPSRC CDT in Modern Statistics and Statistical Machine Learning [EP/S023151/1], a Google PhD Fellowship, and an NSERC Postgraduate Scholarship (PGS D). AS is supported by the EPSRC CDT in Modern Statistics and Statistical Machine Learning [EP/Y034813/1]. MSA is supported by a Junior Fellowship at the Harvard Society of Fellows as well as the National Science Foundation under Cooperative Agreement PHY-2019786 (The NSF AI Institute for Artificial Intelligence and Fundamental Interactions\footnote{http://iaifi.org/}). This work has been made possible in part by a gift from the Chan Zuckerberg Initiative Foundation to establish the Kempner Institute for the Study of Natural and Artificial Intelligence.

%% file: sections/appendix.tex
\section{Proofs}
\label{sec:proofs}
While the main body focuses on the linear schedule, all results extend to a broader class of interpolants. Therefore, here, we use general time schedules $\alpha_t,\beta_t:[0,1]\to\mathbb{R}$ be $C^1$ with $\alpha_t>0$ for $t<1$ and endpoint constraints $\alpha_0=1,\beta_0=0$ and $\alpha_1=0,\beta_1=1$, and define the stochastic interpolant
\begin{equation}
    I_t=\alpha_t I_0+\beta_t I_1,\qquad I_0\sim p_0,\; I_1\sim p_1.
\end{equation}
All results of this section imply the results of the main paper for the choice of $\alpha_{t}=1-t, \beta_{t}=t$.
\subsection{Mean Denoiser}
\label{subsec:proofs_mean_denoiser}

We derive here the form of the mean denoiser stated in~\eqref{eq:mean-denoiser}. Define the schedule-dependent scalars
\begin{equation}
\label{eq:ell_lambda_def}
    \ell_t:=\partial_t \log \alpha_t = \frac{\dot\alpha_t}{\alpha_t},
    \qquad
    \lambda_t:=\dot\beta_t-\beta_t\frac{\dot\alpha_t}{\alpha_t}
    =\dot\beta_t-\beta_t\ell_t.
\end{equation}
For $s<t$, define
\begin{equation}
\label{eq:Gamma_Delta_def}
    \Gamma_{s,t}:=\frac{\alpha_t}{\alpha_s},
    \qquad
    \Xi_{s,t}:=\beta_t-\Gamma_{s,t}\beta_s.
\end{equation}
Then $\Gamma_{s,t}\Gamma_{t,u}=\Gamma_{s,u}$ and $\Xi_{s,t}=\Gamma_{u,t}\Xi_{s,u}+\Xi_{u,t}$ for $s<u<t$.

\paragraph{Mean Denoiser Parametrization.}
We parametrize the flow map by a simplex-valued predictor $\psi_{s,t}:\mathbb{R}^K\to\Delta^{K-1}$ via
\begin{equation}
\label{eq:general_flowmap_param}
    X_{s,t}(x):=\Gamma_{s,t}x+\Xi_{s,t}\psi_{s,t}(x).
\end{equation}
Equivalently, in residual form $X_{s,t}(x)=x+(t-s)v_{s,t}(x)$ we have
\begin{equation}
\label{eq:v_from_psi_general}
    v_{s,t}(x)
    =\frac{\Gamma_{s,t}-1}{t-s}x+\frac{\Xi_{s,t}}{t-s}\psi_{s,t}(x),
    \qquad
    \psi_{s,t}(x)
    =\frac{x+(t-s)v_{s,t}(x)-\Gamma_{s,t}x}{\Xi_{s,t}}.
\end{equation}
If we set $\alpha_{t}=1-t$ and $\beta_t=t$, then this recovers the parameterization in~\eqref{eq:mean_denoiser_param}.

\begin{propbox}
\paragraph{Interpolant Drift.}
The probability-flow drift satisfies
\begin{equation}
\label{eq:general_drift}
    b_t(x) := \mathbb{E}[\dot I_t\mid I_t=x]
    =\ell_t x+\lambda_t\mathbb{E}[I_1\mid I_t=x].
\end{equation}
\end{propbox}
\begin{proof}
Differentiate $I_t=\alpha_t I_0+\beta_t I_1$ to get $\dot I_t=\dot\alpha_t I_0+\dot\beta_t I_1$.
Rewrite $I_0=(I_t-\beta_t I_1)/\alpha_t$ and substitute:
\begin{align}
    \dot I_t
    &=\frac{\dot\alpha_t}{\alpha_t}I_t+\Big(\dot\beta_t-\beta_t\frac{\dot\alpha_t}{\alpha_t}\Big)I_1\\
    &=\ell_t I_t+\lambda_t I_1.
\end{align}
Taking expectation conditional on $I_t = x$ yields~\eqref{eq:general_drift}.
\end{proof}

\begin{propbox}
\paragraph{Diagonal Identity for $\psi_{s,t}$.}
For any $t\in[0,1]$,
\begin{equation}
\label{eq:general_diagonal_identity}
    \psi_{t,t}(x)=\mathbb{E}[I_1\mid I_t=x].
\end{equation}
\end{propbox}
\begin{proof}
From~\eqref{eq:v_from_psi_general},
\begin{align}
    b_t(x)
    &=\lim_{s\to t}v_{s,t}(x)
    =\lim_{s\to t}\Big(\frac{\Gamma_{s,t}-1}{t-s}x+\frac{\Xi_{s,t}}{t-s}\psi_{s,t}(x)\Big)\\
    &=\ell_t x+\lambda_t\psi_{t,t}(x),
\end{align}
using $\lim_{s\to t}\frac{\Gamma_{s,t}-1}{t-s}=\ell_t$ and $\lim_{s\to t}\frac{\Xi_{s,t}}{t-s}=\lambda_t$.
Equating with~\eqref{eq:general_drift} gives~\eqref{eq:general_diagonal_identity}.
\end{proof}

\begin{propbox}
\paragraph{Mean Denoiser.}
Let $(x_u)_{u\in[s,t]}$ be a trajectory of $\dot x_u=b_u(x_u)$. Then
\begin{equation}
\label{eq:general_mean_denoiser_timeavg}
    \psi_{s,t}(x_s)
    =\int_s^t w_{s,t}(u)\,\mathbb{E}[I_1\mid I_u=x_u]\,du,
    \qquad
    w_{s,t}(u):=\frac{\alpha_t}{\Xi_{s,t}}\frac{\lambda_u}{\alpha_u}.
\end{equation}
Moreover $\int_s^t w_{s,t}(u)\,du=1$. If $\lambda_u/\alpha_u\geq 0$ on $[s,t]$ (equivalently $u\mapsto\beta_u/\alpha_u$ is non-decreasing), then $w_{s,t}$ is a probability density and $\psi_{s,t}(x_s)\in\Delta^{K-1}$.
\end{propbox}
\begin{proof}
By~\eqref{eq:general_drift}, the trajectory obeys
\begin{equation}
\label{eq:traj_general}
\frac{d}{du} x_u=\ell_u x_u+\lambda_u\psi_{u,u}(x_u).
\end{equation}
Thus by chain rule
\begin{equation}
\label{eq:integrating_factor_general}
    \frac{d}{du}\Big(\frac{x_u}{\alpha_u}\Big)
=\frac{\ell_u x_u+\lambda_u\psi_{u,u}(x_u)}{\alpha_{u}}-\frac{\dot{\alpha}_u}{\alpha_u^2}x_u
=\frac{\lambda_u}{\alpha_u}\psi_{u,u}(x_u).
\end{equation}
Integrating from $s$ to $t$ and multiplying with $\alpha_{t}$ gives
\begin{equation}
\label{eq:xt_integral_general}
    x_t=\Gamma_{s,t}x_s+\alpha_t\int_s^t \frac{\lambda_u}{\alpha_u}\psi_{u,u}(x_u)\,du.
\end{equation}
Next, note $\frac{d}{du}(\beta_u/\alpha_u)=\lambda_u/\alpha_u$, hence
\begin{equation}
\label{eq:Delta_as_integral}
    \alpha_t\int_s^t \frac{\lambda_u}{\alpha_u}\,du
    =\alpha_t\Big(\frac{\beta_t}{\alpha_t}-\frac{\beta_s}{\alpha_s}\Big)
    =\beta_t-\frac{\alpha_t}{\alpha_s}\beta_s
    =\Xi_{s,t}.
\end{equation}
Normalize the integral term in~\eqref{eq:xt_integral_general} using~\eqref{eq:Delta_as_integral} to obtain~\eqref{eq:general_mean_denoiser_timeavg}. The normalization $\int_s^t w_{s,t}=1$ follows from~\eqref{eq:Delta_as_integral}.
\end{proof}
Now for $\alpha_{t}=1-t,\beta_t=t$, we obtain
\begin{align}
    w_{s,t}(u)=\frac{1-t}{t-\frac{1-t}{1-s}s}\frac{1+u\frac{1}{1-u}}{1-u}=\frac{1-s}{1}\frac{1-t}{t-s}\frac{1+u\frac{1}{1-u}}{1-u}=\frac{1-s}{(1-u)^2}\frac{1-t}{t-s},
\end{align}
recovering the result from the main paper.

\subsection{Flow Map Identities for Mean Denoiser}
We derive the flow map identities in Section~\ref{prop:identities} for general interpolants.
\begin{propbox}
\paragraph{Semigroup Identity for $\psi_{s,t}$.}
For any $s<u<t$,
\begin{equation}
\label{eq:general_semigroup}
    \psi_{s,t}(x)
    =\omega_{s,u,t}\psi_{s,u}(x)+(1-\omega_{s,u,t})\psi_{u,t}(X_{s,u}(x)),
    \qquad
    \omega_{s,u,t}:=\frac{\Gamma_{u,t}\Xi_{s,u}}{\Xi_{s,t}}.
\end{equation}
Moreover $\omega_{s,u,t}+(1-\omega_{s,u,t})=1$, and if $u\mapsto \beta_u/\alpha_u$ is non-decreasing, then $0\leq \omega_{s,u,t}\leq 1$.
\end{propbox}
\begin{proof}
We use the semigroup property of the flowmap: $X_{s,t}=X_{u,t}\circ X_{s,u}$. Using~\eqref{eq:general_flowmap_param},
\begin{align*}
\Gamma_{s,t}x+\Xi_{s,t}\psi_{s,t}(x)
=&X_{s,t}(x)=X_{u,t}(X_{s,u}(x))=\Gamma_{u,t}(\Gamma_{s,u}x+\Xi_{s,u}\psi_{s,u}(x))+\Xi_{u,t}\psi_{u,t}(X_{s,u}(x)).
\end{align*}
Since $\Gamma_{u,t}\Gamma_{s,u}=\Gamma_{s,t}$, we can substract $\Gamma_{s,t} x$ on both sides and obtain
\begin{equation*}
\Xi_{s,t}\psi_{s,t}(x)=\Gamma_{u,t}\Xi_{s,u}\psi_{s,u}(x)+\Xi_{u,t}\psi_{u,t}(X_{s,u}(x)).
\end{equation*}
Divide by $\Xi_{s,t}$ to get~\eqref{eq:general_semigroup} using $\Xi_{s,t}=\Gamma_{u,t}\Xi_{s,u}+\Xi_{u,t}$. If $u\mapsto \beta_{u}/\alpha_u$ is non-decreasing, then \begin{align}
\frac{\beta_t}{\alpha_t}-\frac{\beta_s}{\alpha_s}&\geq 0\\
\Rightarrow\quad \beta_t-\frac{\beta_s}{\alpha_s}\alpha_{t}&\geq 0\\
\Rightarrow\quad \Xi_{s,t}=\beta_t-\Gamma_{s,t}\beta_{s}&\geq 0
\end{align}
This implies that also $\omega_{s,u,t}\geq 0$ and $1-\omega_{s,u,t} =\frac{\Xi_{u,t}}{\Xi_{s,t}} \geq 0$. 
\end{proof}
Note that for $\alpha_{t}=1-t,\beta_t=t$, we get
\begin{align*}
\omega_{s,u,t}=\frac{\Gamma_{u,t}\Xi_{s,u}}{\Xi_{s,t}}
=\frac{\dfrac{1-t}{1-u}\!\left(u-\dfrac{1-u}{1-s}s\right)}{t-\dfrac{1-t}{1-s}s}&=\frac{\dfrac{1-t}{1-u}\cdot\dfrac{u(1-s)-(1-u)s}{1-s}}{\dfrac{t(1-s)-(1-t)s}{1-s}}\\
&=\frac{\dfrac{1-t}{1-u}\cdot\dfrac{u-s}{1-s}}{\dfrac{t-s}{1-s}}\\
&=\frac{1-t}{1-u}\cdot\frac{u-s}{t-s},
\end{align*}
recovering the semigroup identity from the main paper.

\begin{propbox}
\paragraph{Lagrangian Identity for $\psi_{s,t}$.}
For any $s<t$,
\begin{equation}
\label{eq:general_lagrangian}
    \psi_{s,t}(x)=\psi_{t,t}(X_{s,t}(x))-C_{s,t}\partial_t\psi_{s,t}(x),
    \qquad
    C_{s,t}:=\frac{\Xi_{s,t}}{\lambda_t}.
\end{equation}
\end{propbox}
\begin{proof}
Differentiate~\eqref{eq:general_flowmap_param} in $t$:
\begin{align*}
\partial_t X_{s,t}(x)
&=(\partial_t\Gamma_{s,t})x+(\partial_t\Xi_{s,t})\psi_{s,t}(x)+\Xi_{s,t}\partial_t\psi_{s,t}(x),
\end{align*}
where we used that $X_{s,t}(x)=\Gamma_{s,t}x+\Xi_{s,t}\psi_{s,t}(x)$. Use $\partial_t\Gamma_{s,t}=\ell_t\Gamma_{s,t}$ and 
\begin{align*}
\partial_t\Xi_{s,t}=\dot{\beta}_t-\frac{\dot{\alpha}_t}{\alpha_{s}}\beta_s=\dot{\beta}_t-\frac{\dot{\alpha}_t}{\alpha_{t}}\beta_t+\frac{\dot{\alpha}_t}{\alpha_{t}}\beta_t-\frac{\dot{\alpha}_t}{\alpha_{s}}\beta_s=\lambda_t+\ell_t\Xi_{s,t},
\end{align*}
to obtain
\begin{equation}
\label{eq:lag_simplify}
    \partial_t X_{s,t}(x)=\ell_t X_{s,t}(x)+\lambda_t\psi_{s,t}(x)+\Xi_{s,t}\partial_t\psi_{s,t}(x).
\end{equation}
By tangency, $\partial_t X_{s,t}(x)=b_t(X_{s,t}(x))=\ell_t X_{s,t}(x)+\lambda_t\psi_{t,t}(X_{s,t}(x))$.
Equate with~\eqref{eq:lag_simplify} and rearrange to get~\eqref{eq:general_lagrangian}.
\end{proof}

\begin{propbox}
\paragraph{Eulerian Identity for $\psi_{s,t}$.}
For any $s<t$,
\begin{equation}
\label{eq:general_eulerian}
    \partial_s\psi_{s,t}(x)+J_x\psi_{s,t}(x)b_s(x)
    =\kappa_{s,t}\big(\psi_{s,t}(x)-\psi_{s,s}(x)\big),
    \qquad
    \kappa_{s,t}:=\frac{\Gamma_{s,t}\lambda_s}{\Xi_{s,t}}.
\end{equation}
\end{propbox}
\begin{proof}
Use $\partial_s X_{s,t}(x)+J_x X_{s,t}(x)b_s(x)=0$ and~\eqref{eq:general_flowmap_param}.
Compute $\partial_s\Gamma_{s,t}=-\ell_s\Gamma_{s,t}$ and $\partial_s\Xi_{s,t}=-\Gamma_{s,t}\lambda_s$.
Then
\begin{align}
0&=-\ell_s\Gamma_{s,t}x-\Gamma_{s,t}\lambda_s\psi_{s,t}(x)+\Xi_{s,t}\partial_s\psi_{s,t}(x)
+\Gamma_{s,t}b_s(x)+\Xi_{s,t}J_x\psi_{s,t}(x)b_s(x).
\end{align}
Using $b_s(x)=\ell_s x+\lambda_s\psi_{s,s}(x)$ cancels the $\ell_s$ terms and yields
\begin{equation}
\Xi_{s,t}\big(\partial_s\psi_{s,t}(x)+J_x\psi_{s,t}(x)b_s(x)\big)
=\Gamma_{s,t}\lambda_s\big(\psi_{s,t}(x)-\psi_{s,s}(x)\big),
\end{equation}
which gives~\eqref{eq:general_eulerian}.
\end{proof}

\paragraph{Logit Consistency.}
To derive an equivalent formulation of the consistency identities that is well suited to optimization, we work in logit space and represent
\begin{equation}
    \psi_{s,t}(x)=\mathrm{Softmax}(z_{s,t}(x)).
\end{equation}
\begin{propbox}
\paragraph{Lagrangian Logit Consistency.}
The Lagrangian identity~\eqref{eq:general_lagrangian} is equivalent to
\begin{equation}
\label{eq:general_lagrangian_logit}
    \psi_{s,t}(x)
    =\text{Softmax}\Big(
    z_{t,t}(X_{s,t}(x))
    -\log\big(\mathbf{1}+C_{s,t}(\partial_t z_{s,t}(x)-\langle \psi_{s,t}(x),\partial_t z_{s,t}(x)\rangle\mathbf{1})\big)
    \Big).
\end{equation}
\end{propbox}
\begin{proof}
Differentiate $\psi_{s,t} =\text{Softmax}(z_{s,t})$ in $t$ to get
\begin{equation}
\partial_t\psi_{s,t}=\psi_{s,t}\odot(\partial_t z_{s,t}-\langle \psi_{s,t},\partial_t z_{s,t}\rangle\mathbf{1}).
\end{equation}
Substitute into~\eqref{eq:general_lagrangian} and rearrange elementwise to get
\begin{equation}
    \psi_{t,t} \circ X_{s,t}=\psi_{s,t}\odot(\mathbf{1}+C_{s,t}(\partial_t z_{s,t}-\langle \psi_{s,t},\partial_t z_{s,t}\rangle\mathbf{1})).
\end{equation}
Taking elementwise $\log$ and using $\log\psi_{s,t}=z_{s,t}+c_{s,t}\mathbf{1}$ for some scalar $c_{s,t}$ yields~\eqref{eq:general_lagrangian_logit}.
\end{proof}

\begin{propbox}
\paragraph{Eulerian Logit Consistency.}
The Eulerian identity~\eqref{eq:general_eulerian} is equivalent to
\begin{equation}
\label{eq:general_eulerian_logit}
    \psi_{s,t}(x)
    =\text{Softmax}\Big(
    z_{s,s}(x)
    -\log\big(\mathbf{1}-\kappa_{s,t}^{-1}(D_s z_{s,t}(x)-\langle \psi_{s,t}(x),D_s z_{s,t}(x)\rangle\mathbf{1})\big)
    \Big),
\end{equation}
where $D_s z_{s,t}(x):=\partial_s z_{s,t}(x)+J_x z_{s,t}(x)b_s(x)$.
\end{propbox}
\begin{proof}
Differentiating along $b_s$, we have
\begin{equation}
    D_s\psi_{s,t}=\psi_{s,t}\odot(D_s z_{s,t}-\langle \psi_{s,t},D_s z_{s,t}\rangle\mathbf{1}).
\end{equation}
Substitute into~\eqref{eq:general_eulerian} and rearrange:
\begin{equation}
    \psi_{s,s}=\psi_{s,t}\odot\big(\mathbf{1}-\kappa^{-1}(D_s z_{s,t}-\langle \psi_{s,t},D_s z_{s,t}\rangle\mathbf{1})\big).
\end{equation}
Taking elementwise $\log$ and using $\log\psi_{s,t}=z_{s,t}+c_{s,t}\mathbf{1}$ yields~\eqref{eq:general_eulerian_logit}.
\end{proof}

\subsection{Training objectives for Mean Denoiser}

\paragraph{Diagonal Loss.}
By~\eqref{eq:general_diagonal_identity}, we can train $\hat\psi_{t,t}$ with cross-entropy under the general interpolant:
\begin{equation}
\label{eq:general_Ldiag}
    \mathcal{L}_{\text{diag}}(\hat\psi)
    =\int_0^1 \mathbb{E}\Big[
    -\sum_{k=1}^K I_1^{(k)}\log \hat\psi_{t,t}^{(k)}(I_t)
    \Big]dt,
    \qquad
    I_t=\alpha_t I_0+\beta_t I_1.
\end{equation}

\paragraph{Model Flow Map.}
Given $\hat\psi$, define
\begin{equation}
\label{eq:general_model_flowmap}
    \hat X_{s,t}(x):=\Gamma_{s,t}x+\Xi_{s,t}\hat\psi_{s,t}(x),
    \qquad
    \hat b_s(x):=\ell_s x+\lambda_s\hat\psi_{s,s}(x).
\end{equation}

\paragraph{Consistency via Semigroup Loss (PSD).}
Enforce~\eqref{eq:general_semigroup} by KL distillation with teacher
$T_{\text{PSD}}:=\omega_{s,u,t}\hat\psi_{s,u}(I_s)+(1-\omega_{s,u,t})\hat\psi_{u,t}(\hat X_{s,u}(I_s))$:
\begin{equation}
\label{eq:general_Lpsd}
    \mathcal{L}_{\text{PSD}}(\hat\psi)
    =\iiint_{0\leq s\leq u\leq t\leq 1}
    \mathbb{E}\Big[
\KL\big(\textrm{sg}[T_{\text{PSD}}]\|\hat\psi_{s,t}(I_s)\big)
    \Big]ds\,du\,dt.
\end{equation}

\paragraph{Consistency via Lagrangian Loss (LSD).}
Using~\eqref{eq:general_lagrangian_logit}, define the teacher
\begin{equation}
\label{eq:general_teacher_lsd}
    T_{\text{LSD}}
    :=\text{Softmax}\Big(
    \textrm{sg}\big[
    \hat z_{t,t}(\hat X_{s,t}(I_s))
    -\log\big(\mathbf{1}+C_{s,t}(\partial_t\hat z_{s,t}(I_s)-\langle \hat \psi_{s,t}(I_s), \partial_t \hat z_{s,t}(I_s) \rangle\mathbf{1})\big)
    \big]
    \Big),
\end{equation}
where $\hat\psi_{s,t}=\text{Softmax}(\hat z_{s,t})$. Minimize
\begin{equation}
\label{eq:general_Llsd}
    \mathcal{L}_{\text{LSD}}(\hat\psi)
    =\iint_{0\leq s\leq t\leq 1}
    \mathbb{E}\Big[
\KL\big(T_{\text{LSD}}\|\hat\psi_{s,t}(I_s)\big)
    \Big]ds\,dt.
\end{equation}

\paragraph{Consistency via Eulerian Loss (ESD).}
Using~\eqref{eq:general_eulerian_logit}, define
$D_s\hat z_{s,t}(x):=\partial_s\hat z_{s,t}(x)+J_x\hat z_{s,t}(x)\hat b_s(x)$. The teacher is
\begin{equation}
\label{eq:general_teacher_esd}
    T_{\text{ESD}}
    :=\text{Softmax}\Big(
    \textrm{sg}\big[
    \hat z_{s,s}(I_s)
    -\log\big(\mathbf{1}-\kappa_{s,t}^{-1}(D_s\hat z_{s,t}(I_s)-\langle \hat\psi_{s,t}(I_s),D_s\hat z_{s,t}(I_s)\rangle\mathbf{1})\big)
    \big]
    \Big),
\end{equation}
and the loss is
\begin{equation}
\label{eq:general_Lesd}
    \mathcal{L}_{\text{ESD}}(\hat\psi)
    =\iint_{0\leq s\leq t\leq 1}
    \mathbb{E}\Big[
    \KL\big(T_{\text{ESD}}\|\hat\psi_{s,t}(I_s)\big)
    \Big]ds\,dt.
\end{equation}

\paragraph{Reduction to the Linear Schedule.}
For $\alpha_t=1-t$ and $\beta_t=t$, we have $\Gamma_{s,t}=\frac{1-t}{1-s}$,
$\Xi_{s,t}=\frac{t-s}{1-s}$, $\lambda_t=\frac{1}{1-t}$, and hence
$C_{s,t}=\frac{(t-s)(1-t)}{1-s}$ and $\kappa_{s,t}=\frac{1-t}{(1-s)(t-s)}$,
recovering the linear interpolant coefficients and identities.

\section{Experimental Details}
\label{sec:experimental_details}
In this section, we present further discussion of our experimental settings, and any hyper-parameters. 

\paragraph{Training Details.} We use a batch size of 512 for both the diagonal and distillation stages. We use 2500 warm-up steps and then a constant learning rate of $3\times10^{-4}$. We use the Adam optimizer~\citep{kingma2017adam} with $\beta_1 = 0.9$ and $\beta_2 = 0.999$. Below, we present the hyper-parameters that are specific to different set-ups.

\paragraph{Diagonal.} For diagonal training, we leverage the adaptive loss presented in~\eqref{eq:loss_adaptive} using $r=0.5, c=0.01$.

\paragraph{PSD.} We found it beneficial to use \textit{gradient surgery}  \citep{yu2020gradientsurgerymultitasklearning, zhong2026riemannianmeanflowonestepgeneration}, commonly used for stabilizing multi-objective optimization, when training with PSD. Concretely, we give priority to the gradient of the diagonal loss, and project out the component of the distillation-loss gradient that conflicts with it. Without this gradient projection, we found that optimization is susceptible to collapsing toward degenerate solutions, at the detriment of the diagonal loss and the overall learning procedure. For PSD, we also found it beneficial to use a learnable loss weighting as a function of only $s$, and not both $(s,t)$ as in~\cite{boffi2025consistency}.

\paragraph{ESD.} For ESD, we do not leverage gradient surgery, as it was sufficiently stable without it, and use a learnable loss weighting that is a function of both $(s,t)$.

\paragraph{Noise Schedule.} We found it empirically beneficial to modify the argmax schedule described in the main body (Section~\ref{sec:time_reparametrisation}). Specifically, we take a convex combination $\tilde \beta(t):=\lambda\beta(t) + (1-\lambda)t$, and use $\lambda=0.9$. Intuitively, this combination spends more time in low and high noise regions (near 0 and 1), than the argmax schedule, which is sharper at the endpoints. 

\section{Example Generations}
\label{sec:example_generations}
\begin{figure}[t]
\centering
\begin{figurepanel}

\noindent
\begin{minipage}[t]{0.485\linewidth}
\nfeheader{NFE = 2}

\vspace{8pt}

\begin{samplebox}
\noindent
\begin{minipage}[t]{0.56\linewidth}
\methodtitle{DFM (ESD)}
\end{minipage}\hfill
\begin{minipage}[t]{0.40\linewidth}
\raggedleft \metrics{53.88}{4.05}
\end{minipage}

\vspace{4pt}
\noindent
[CLS] at the time they said they were the foundation for enough children for the missing school. [CLS] the only news is, despite the fact that i still not be seen enough. [CLS] \ " at the same time, the club is often a smallon, but i am m going to have to [CLS] we give our back to our country. [CLS] an he said he ' s now home. [CLS] but i didn ' t take much of that : him to look. [CLS] outside only the 1. 5 down to \$ 6. 2 billion a head. [CLS] next night ' s hearing will be the last of other measures as well to the government [CLS]
\end{samplebox}

\vspace{8pt}

\begin{samplebox}
\noindent
\begin{minipage}[t]{0.56\linewidth}
\methodtitle{DFM (PSD)}
\end{minipage}\hfill
\begin{minipage}[t]{0.40\linewidth}
\raggedleft \metrics{83.64}{4.10}
\end{minipage}

\vspace{4pt}
\noindent
[CLS] what to do they said they want that is for to children for the afghan people. [CLS] the only news is, though the fact that a still not was seen enough. [CLS] \ " at the same time, the weekend is a comhenon, but in that is a major from across on a continent and in the whole country. [CLS] but he said he ' s now home. [CLS] the remit alsos take place yesterday that expected him to look eastwards later than the 1. 5 down to \$ 6. 2 billion were killed. [CLS] so there ' s a return on the advice of other president is well to the government [CLS]
\end{samplebox}

\end{minipage}
\hfill
\begin{minipage}[t]{0.485\linewidth}
\nfeheader{NFE = 1024}

\vspace{8pt}

\begin{samplebox}
\noindent
\begin{minipage}[t]{0.56\linewidth}
\methodtitle{DFM (ESD)}
\end{minipage}\hfill
\begin{minipage}[t]{0.40\linewidth}
\raggedleft \metrics{32.84}{4.07}
\end{minipage}

\vspace{4pt}
\noindent
[CLS] at the time they said they were not caring for enough children for the public school. [CLS] the only news is, despite the fact that i still not be seen enough. [CLS] \ " at the same time, the economy is getting a smallon, but i ' m going to have to on a plane and back to this country. [CLS] but he said he ' s returning home. [CLS] \ " i don ' t think much of that we want to look at it rather than the details about half of that, \ " he told the bbc news. [CLS] next tuesday ' s hearing will be the last of other important issues well to the government [CLS]
\end{samplebox}

\vspace{8pt}

\begin{samplebox}
\noindent
\begin{minipage}[t]{0.56\linewidth}
\methodtitle{DFM (PSD)}
\end{minipage}\hfill
\begin{minipage}[t]{0.40\linewidth}
\raggedleft \metrics{42.09}{4.05}
\end{minipage}

\vspace{4pt}
\noindent
[CLS] at the time they said they were the foundation for enough children for the haitian people. [CLS] the only news is, despite the fact that i still not be seen enough. [CLS] \ " at the same time, the economy is often a smallon, but i ' m going to have gone on and give and back to this country. [CLS] and he said he ' s now home. [CLS] but i didn ' t take much of that of him to look. [CLS] gasoline prices gained 1. 5 cents to \$ 6. 25 a gallon. [CLS] and there ' s a chance that the presence of other candidates is likely to be re [CLS]
\end{samplebox}
\end{minipage}

\end{figurepanel}
\caption{Example LM1B generations using 2 and 1024 function evaluations.}
\label{fig:samples}
\end{figure}
\clearpage

\begin{figure}[t]
\centering
\begin{figurepanel}
\noindent
\nfeheader{$w=0.0$}\hfill
{\color{muted}\shortstack[r]{Gen.\ PPL: \textbf{51.92}\\Entropy: \textbf{5.22}}}
\vspace{8pt}

\begin{samplebox}
\color{juniper!85!black} [END] For making sure you have a good in mind areas of Warcraft, the game may also be pre-free in the first couple of years, with the first option being Bald NPC’s, as well as a maxed player’s added as a bonus bonus. You don’t need much tweaking when you’re ready to played. Embed this thing at j@@@gmail.com for updates. Story continues below ...[END] The former U.S. Army general was accused of being reunited with his Facebook colleagues in Washington next year as a way to secure place in the Middle East, the possibility of Americans using information and a VPN. “I haven’t about it a long time,” Carson while he was in for an interview, said. “It’s kind of just about fear that it’s going to be amazing how many people are going about being in the country,” he said. “We actually knew that these people are going to be working and not everything going – we sat down, because we said ‘no,’ OK.” Carson and the U.S.-led relationship discussed the allegations\color{black}[END]  to police and there is a crack at the ‘right’ in life for drug reasons. It’s interesting because they were here to find out what happened with this incident.” Following the release of the vote, Prime Minister of Canada warned that a real threat was simply the worst part of evidence the government could issue. “The night we saw it was a great incident as how the military was done,” Mr. said Wednesday. “We are saying that the people of the province and the people whose livelihood is paramount. So do what they looked at things that had been approved by the police department, we put them to stand up on the world’s… largest government. “This is the beginning … It is the end of the 1.5th Amendment, which in effect guarantees the right to provide the government accountable for local businesses,” Mr. Mulcair said. The 1.5 starts at \$15million in a 1 p.m. curfew and \$1 hour for Friday in Toronto, one day starting in Toronto and a four-day stay for Thursday.[END] It will no longer be about Clinton or Sanders attending a Democratic rally in Brooklyn, N.Y\color{juniper!85!black} [END] them on winning the truth and with your liker. They, I was reminded of Huntsman’s argument: “When they win the general Cup, when they ask a medal if that’s not good if they love, and if they think they like, I believe, then that’s great. They could earn more votes. I like right your man.” But that, too, is part of Huntsman’s tendency to say “no” what’s best or what lurk. That he doesn’t, and so, what makes I’m afraid “They can’t preparer if they’re just going to be the butt of them to pay us about this and other reason.” But its I mean, as Jesus said and saying, you love the rich and the fans and all our godly invested in those things, talking about what we do with them– and of course else. That is because there is a self God’s just at, and I think it’s okay. And I’m paraphr and for the most part, to go to a foreign city and my country\color{black} [END] help us about the extent of how people feel more comfortable and socialized. Currently we have no data to create a personal model that is part of society. At the moment, suppose you know a consumer what you are putting into its individual lives. It is not necessarily its own identity. However, you are not quite an individual individual, and you can just choose to look it to you. You have a choice if you already know anything. That would be your own evidence. In our response to the very basic associated with our lives, there is no doubt at all that you are in a world where you own much of our wealth. It is a responsibility to society and that is determined by the individual who is more responsible. We are also at a higher risk for the environment in which we live. That is what we make our decisions. You can come up with some work and some money, but less likely depending on time you were for. We probably might initially say this is a our dream, but it takes the actual work from our perspective. Because the participants are at a loss for a ROI or ROI income, that the average one (below) is more valuable. In other words, you can mean the
\end{samplebox}

\end{figurepanel}
\caption{Example OWT generation using CFG ($\omega=0.0$) and block sampling to generate 4 blocks of size 256. A change in text colour denotes a new block.}
\label{fig:samples_owt_0}
\end{figure}

\newpage

\begin{figure}[t]
\centering
\begin{figurepanel}
\noindent
\nfeheader{$w=1.0$}\hfill
{\color{muted}\shortstack[r]{Gen.\ PPL: \textbf{41.32}\\Entropy: \textbf{5.10}}}
\vspace{8pt}
\begin{samplebox}
\color{juniper!85!black} 
[END] For making sure you have a good in mind areas of Warcraft, the game may also be pre-free in the first couple of years, with the first option being Bald NPC’s, as well as a maxed player’s added as a bonus bonus. You don’t need much tweaking when you’re ready to played. Embed this thing at j@@@gmail.com for updates. Story continues below ...[END] The former U.S. Army general was accused of being reunited with his Facebook colleagues in Washington next year as a way to secure place in the Middle East, the possibility of Americans using information and a VPN. “I haven’t about it a long time,” Carson while he was in for an interview, said. “It’s kind of just about fear that it’s going to be amazing how many people are going about being in the country,” he said. “We actually knew that these people are going to be working and not everything going – we sat down, because we said ‘no,’ OK.” Carson and the U.S.-led relationship discussed the allegations \color{black}  between the U.S. team in Washington. “It’s been a far while quiet,” Creamer said, with calls to him out. “Very yeah.” Given the nature of the interview, Carson said it was unlikely that a real chance of having the Russians part of a foreign government could occur. “The media we use it was a great distraction as how the information was leaked,” he said during interviews. “We are saying that a lot of the public and the people whose trust is compromised. So do what they are at (it is) now very clearly and say, ‘100 percent,'” he said. “That’s bad for us, OK.’ So this is the fact that the 1 will eventually give each night.” “The Russians are in good all away,” Mr. Carson said. “A lot of our servers are being down, and our group is only beginning.” “The fact that ISIS had only been in Iraq and Syria, at some point where the leaks came out, appears to be about Clinton or Bernie Sanders,” the U.S. Army \color{juniper!85!black} said. The general and his colleagues like Bitcoin. They only work 100 percent of his wife’s salary if they get the benefit of the 1,000 before they leave. “We are not talking about one guy,” said Mr. Carson, laughing. “It’s great so we could build more quickly. If we vote against it.” But that, too, is part of Cadman’s willingness to talk. “There’s really lots of things about the U.S. talk about, and so, what happened,” another veteran said in Washington. “And people are starting to think they’re just going to be the butt of them to pay us about this and other stuff.” In its closing interview, Facebook News said and saying that you trust the soldier and the one with all our godly confidence in those conversations, thinking about what to do. “A U.S.” is a self Carson’s looking at, and closer to future, in dealing with Mr. Creamer.[END] Dmaine was one of the most popular places to go to a Hillary Clinton and my country \color{black} is very concerned about the importance of being around their medical education and social care. This we have no longer anymore – a national reality that is now becoming forgotten. At the talk, recent about the phrase “we are her into its hands” has some Americans’ own reactions. However, after I was Hillary’s second and voted out of the United States Senate, I still consider it a far-flung presidential campaign. A wayDmaine In our countdown to the 10th president, our nation is there from the U.S. we are in a country where the country was before our election. It is a testament to what the U.S. has taught us since this year. We are also at a loss risk for the environment in which we live. That is what we make our country. After we set up with some work and some money, the less likely vote on U.S. soil like the 8-Ele Partnership is a good reason, but it appears the economy has hit our economy. And the Democrats are at a loss for one of their candidates, Bernie Sanders, is the biggest candidate left in their most against Hillary. In today’s election[END]
\end{samplebox}

\end{figurepanel}
\caption{Example OWT generation using CFG ($\omega=1.0$) and block sampling to generate 4 blocks of size 256. A change in text colour denotes a new block.}
\label{fig:samples_owt_1}
\end{figure}

\newpage

\begin{figure}[t]
\centering
\begin{figurepanel}
\noindent
\nfeheader{$w=2.0$}\hfill
{\color{muted}\shortstack[r]{Gen.\ PPL: \textbf{28.41}\\Entropy: \textbf{4.93}}}
\vspace{8pt}

\begin{samplebox}
\color{juniper!85!black}[END] For making sure you have a good in mind areas of Warcraft, the game may also be pre-free in the first couple of years, with the first option being Bald NPC’s, as well as a maxed player’s added as a bonus bonus. You don’t need much tweaking when you’re ready to played. Embed this thing at j@@@gmail.com for updates. Story continues below ...[END] The former U.S. Army general was accused of being reunited with his Facebook colleagues in Washington next year as a way to secure place in the Middle East, the possibility of Americans using information and a VPN. “I haven’t about it a long time,” Carson while he was in for an interview, said. “It’s kind of just about fear that it’s going to be amazing how many people are going about being in the country,” he said. “We actually knew that these people are going to be working and not everything going – we sat down, because we said ‘no,’ OK.” Carson and the U.S.-led relationship discussed the allegations\color{black}’ claims and others at a press conference. “It’s been a a long time,” Creamer said in the interview to him out. “Prior to Carson, after he was in the middle of the scandal, I’d realized that a real job wasn’t part of a secret government operation.  Photo  “People thought it was a funny thing as opposed just in the beginning,” he said during interview.  “We are saying that a lot of the public and the people whose office is going to actually do what they are hiding (which is going to happen) and say, ‘Oh no,'” he said. “That’s scary for us, and there’s no longer the fact that we’re still at night.”  “The allegations are just been years away,” Mr. Carson said. “A lot of people would just sit down, and this group could create jobs.”  “The fact that Hillary had only been in office and came in November – and I am confident that we will continue to be about Clinton or a Sanders,” the U.S. Army \color{juniper!85!black} said during interviews in the interview at his headquarters on Monday. “There is question that I wouldn’t know if they’ve been the target of the scandal.”  “We are not talking about one party, and this is a major change, because we also knew that’s great so we could build more jobs. If we chimed in.”  But back to November, Carson said he expected Clinton’s emails to talk.  “There’s really lots of things about the U.S. Army about now and beyond,” Clinton’s office said in August. “These people are told now that they’re just going to be too scared of them to pay us about Obama and a Sanders.”  But Creamer said that as Carson said and Creamer talked about the scandal and the one that grew underwhelmingly by his aides, they talked about what to do.  “A Mr.S.” is a self he’s looking at, and hoping to share, in interviews with Mr. Carson. He said he had “one of the most difficult places to go to work alongside Clinton and her country\color{black}.”  If Carson’s Twitter account and a potential email account came to him, he said he wants a Facebook network that has now allowed Americans to share information and use.  “We really don’t know anything about the media, we don’t know the truth, look at these guy’s even gotten you out of there,” he said. “We don’t know anything. That would be your Mr. Why would we go to the media?”  Benarson said the U.S. Army, including Mr. Carson’s emails, and that he used a case to get the U.S. Army in a bind this year.  “There is a legal basis for that?” he said. “Then we make our country.”  Photo  He also said that he would not say if U.S. forces might be targeting military service overseas or who are involved, including Congress, the War on Drugs and Syria.   Creamer said that at a press conference, as she announced many weeks ago, that the State Department would release their findings to all evidence of Mr. Carson’s emails[END]
\end{samplebox}

\end{figurepanel}
\caption{Example OWT generation using CFG ($\omega=2.0$) and block sampling to generate 4 blocks of size 256. A change in text colour denotes a new block.}
\label{fig:samples_owt_2}
\end{figure}

\newpage

\begin{figure}[t]
\centering
\begin{figurepanel}
\noindent
\nfeheader{$w=3.0$}\hfill
{\color{muted}\shortstack[r]{Gen.\ PPL: \textbf{29.74}\\Entropy: \textbf{4.95}}}
\vspace{8pt}

\begin{samplebox}
\color{juniper!85!black}[END] For making sure you have a good in mind areas of Warcraft, the game may also be pre-free in the first couple of years, with the first option being Bald NPC’s, as well as a maxed player’s added as a bonus bonus. You don’t need much tweaking when you’re ready to played. Embed this thing at j@@@gmail.com for updates. Story continues below ...[END] The former U.S. Army general was accused of being reunited with his Facebook colleagues in Washington next year as a way to secure place in the Middle East, the possibility of Americans using information and a VPN. “I haven’t about it a long time,” Carson while he was in for an interview, said. “It’s kind of just about fear that it’s going to be amazing how many people are going about being in the country,” he said. “We actually knew that these people are going to be working and not everything going – we sat down, because we said ‘no,’ OK.” Carson and the U.S.-led relationship discussed the allegations\color{black}’ claims and others at a press conference. “It’s been a a long time,” Creamer said in an email to him out. “Sad to say, after he was in the middle of the Army, I’ve realized that a real VPN wasn’t part of a cyber operation against me.” “People thought it was a great thing as everybody else in the world,” he said during interview. “We are afraid that a lot of the public and the people whose information is going to actually do what they are doing (‘ no,’ and say, ‘no no,'” he said. “That’s scary for us, and there’s no escaping the fact that we’re still ‘how are this going to be a message we are going to walk away,” Mr. Carson said. “A lot of people would just sit down, and this thing could create jobs.” “The fact that these had only been in service and a VPN is – and I think we found we knew you to be about privacy or a VPN,” the U.S. Army \color{juniper!85!black}  general said in a statement shortly before his interview on Facebook. “There was question that I wouldn’t know if they’re being in the world, that they are a VPN if that’s not – if you knew, and this is a bad idea, we believe, then that’s exactly what we could do more easily. If we looked into it.” Carson has posted a tough attack on North Korea’s willingness to talk. “It’s really kind of just that the U.S. cares about, and so, what makes people’s lives” Carson said. “Some people had told them that they’re just going to be the butt of them to pay us about privacy and a VPN.” But Creamer said that as Americans approached and share their stories about the allegations and the fear that people have gone online with his colleagues, they worry about what to do. “A U.S.” is a self-described millionaire who from a decade and goes to future work in Washington. “He said he had always was one of the most dangerous places to go to a cyber station and my country \color{black} is very concerned about the dangers of how people feel,” he said. Carson said he had made it a long ago that he now knows how to read information and use. “We really didn’t make assumptions about the system, we didn’t pay the bills, and I just didn’t even tell that really thing us,” he said. “We didn’t know anything. That would be your toaster. In our world.” According to Carson, Carson and the U.S. Army are at odds. They’re hoping to get their stories through a portal to what the U.S. did in Washington earlier this year. “How is it with them for that?” he said. “They can make our country. You can work up with them.” A less likely asked on U.S. Facebook policy about his military in Washington is a good reason, how he and the U.S told him. Camper said that at a dinner for a meeting five or many weeks ago, that conversation with them would in their office to see what he or each other’s name[END]
\end{samplebox}

\end{figurepanel}
\caption{Example OWT generation using CFG ($\omega=3.0$) and block sampling to generate 4 blocks of size 256. A change in text colour denotes a new block.}
\label{fig:samples_owt_3}
\end{figure}